\documentclass[12pt]{colt2017}

\usepackage[utf8]{inputenc}
\usepackage{amsmath,amscd,amssymb}%,amsthm}
\usepackage{verbatim}
\usepackage{thmtools}
\usepackage{thm-restate}

\declaretheorem[name=Theorem]{Theorem}
\declaretheorem[name=Proposition, numberlike=Theorem]{Proposition}
\declaretheorem[name=Lemma, numberlike=Theorem]{Lemma}
\declaretheorem[sibling=Theorem]{Definition}
\declaretheorem[sibling=Theorem]{Corollary}

\usepackage[framemethod=TikZ]{mdframed}
\mdfdefinestyle{MyFrame}{%
    linecolor=black,
    outerlinewidth=2pt,
    roundcorner=20pt,
    innerrightmargin=20pt,
    innerleftmargin=20pt,
    backgroundcolor=white}
\usepackage{natbib}
\usepackage{hyperref}

\newcommand{\freerex}{\textsc{FreeRex}}

\newcommand{\R}{\mathbb{R}}
\newcommand{\E}{\mathop{\mathbb{E}}}
\newcommand{\argmin}{\mathop{\text{argmin}}}
\newcommand{\argmax}{\mathop{\text{argmax}}}

\newcommand{\Lm}{L_{\max}}

\newcommand{\sigmamin}{\sigma_{\text{min}}}

\usepackage{graphicx}

\usepackage{algorithm}
\usepackage{algorithmic}

\title[Online Learning Without Prior Information]{Online Learning Without Prior Information}

\coltauthor{\Name{Ashok Cutkosky} \Email{ashokc@cs.stanford.edu}\AND
 \Name{Kwabena Boahen} \Email{boahen@stanford.edu}\\
  \addr Stanford University}
  
\usepackage{times}
\begin{document}

\maketitle

\begin{abstract}
The vast majority of optimization and online learning algorithms today require
some prior information about the data (often in the form of bounds on gradients or on the optimal parameter value). When this information is not available, these algorithms require laborious manual tuning of various hyperparameters, motivating the search for algorithms that can adapt to the data with no prior information. We  describe a frontier of new lower bounds on the performance of such algorithms, reflecting a tradeoff between a term that depends on the optimal parameter value and a term that depends on the gradients' rate of growth. Further, we construct a family of algorithms whose performance matches any desired point on this frontier, which no previous algorithm reaches.
\end{abstract}

\section{Problem Definition and Prior Work}

Data streams, large datasets, and adversarial environments require online optimization algorithms, which continually adapt model parameters to the data. At iteration $t$, these algorithms pick a point $w_t\in W$, are presented with a loss function $\ell_t:W\to \R$, and suffer loss $\ell_t(w_t)$. The algorithm's performance is measured by \emph{regret}, which is defined as the loss relative to some comparison point $u$:
\begin{align*}
    R_T(u)&= \sum_{t=1}^T \ell_t(w_t)-\ell_t(u)
\end{align*}
When $W$ is a convex set and the $\ell_t$ are guaranteed to be convex, the regret can be minimized using only information about the gradients of $\ell_t$ at $w_t$, leading to simple and efficient algorithms.

All online convex optimization algorithms require either a bound $B$ on the diameter of $W$ or a bound $\Lm$ on the gradients of $\ell_t$, or suffer a penalty that is exponential in gradients' rate of growth when no information is given  \citep{cutkosky2016online}. When $B$ is known but $\Lm$ is unknown, there are algorithms that can obtain regret $O(B\Lm\sqrt{T})$ \citep{duchi10adagrad,mcmahan2010adaptive}. Conversely, when $B$ is infinite (e.g. $W$ is an entire vector space) but $\Lm$ is known, there are algorithms that obtain $O(\|u\|\Lm\sqrt{T\log(\|u\|T)})$ or $O(\|u\|\Lm\sqrt{T}\log(\|u\|T))$ regret \citep{mcmahan2012no,orabona2013dimension, mcmahan2013minimax,orabona2014simultaneous, orabona2016coin}. The situation does not improve when both $B$ and $\Lm$ are known. In this case it is impossible to do better than $O(B\Lm \sqrt{T})$, so knowing just one of these parameters is essentially as good as knowing both \citep{abernethy2008optimal}. In the case where no prior information is given, it was recently proved by \citet{cutkosky2016online} that the regret must contain an additional exponential penalty $\exp(\max_t\sqrt{L_t/L_{t-1}})$,
where $L_t$ is the maximum gradient observed by iteration $t$.

The case in which we have no bound on either $B$ or $\Lm$ is common in practice. A standard pragmatic approach to this lack of information is to simply make a guess for these parameters and then apply an algorithm that uses the guess as input, but this approach is theoretically unsound in online learning, and rather laborious and inelegant in general. We explore lower bounds and algorithms that adapt to the unknown quantities in a principled way in this paper.

Where no information is given, we prove that there is a frontier of matching lower and upper bounds on $R_T(u)$ that trades-off a $\|u\|\Lm\sqrt{T}\log(\|u\|T)$ term with a $\exp(\max_tL_t/L_{t-1})$ term along two dimensions, which we parametrize by $k$ and $\gamma$.\footnote{The square root is missing from the exponential term because we improved the lower bound given in \citet{cutkosky2016online} (see Section \ref{sec:lowerbound}).} Along the first dimension, the exponential penalty is reduced to $\exp((L_t/L_{t-1})/k^2)$ for any $k>0$ at the expense of rescaling the regret's $\sqrt{T}$ term to $k\|u\|\Lm\sqrt{T}\log(\|u\|T)$. Along the second dimension, the logarithm's power in the $\sqrt{T}$ term is reduced to $\|u\|\Lm\sqrt{T}\log^\gamma(\|u\|T)$ for any $\gamma\in (1/2,1]$ at the expense of increasing the exponential penalty to $\exp((L_t/L_{t-1})^{1/(2\gamma-1)})$. We prove the lower bounds by constructing a specific adversarial loss sequence, and we prove the upper bounds by providing a family of algorithms whose regret matches the lower bound frontier for any $k$ and $\gamma$.

\section{Notation and Setup}\label{sec:background}
Before proceeding further, we provide a few definitions that will be useful throughout this paper. A set $W$ is a \emph{convex set} if $W$ is a subset of some real vector space and $tx+(1-t)y\in W$ for all $x,y\in W$ and $t\in[0,1]$. Throughout this paper we will assume that $W$ is closed. A function $f$ is a \emph{convex function} if $f(tx+(1-t)y)\ge tf(x)+(1-t)f(y)$ for all $x,y$ and $t\in[0,1]$. If $f:V\to \R$ for some vector space $V$, then a vector $g\in V^*$ is a \emph{subgradient} of $f$ at $x$, denoted $g\in \partial f(x)$, if $f(y)\ge f(x)+g\cdot (y-x)$ for all $y$. Here we use the dot product to indicate application of linear functionals in the dual space since this should cause no confusion. A \emph{norm} $\|\cdot\|$ is a function such that $\|x\|=0$ if and only if $x=0$, $\|cx\|=|c|\|x\|$ for any scalar $c$, and $\|x+y\|\le \|x\|+\|y\|$ for all $x$ and $y$. The \emph{dual norm} is a norm $\|\cdot\|_\star$ defined by $\|x\|_\star = \sup_{\|y\|=1}x\cdot y$. As a special case, when $\|x\|=\sqrt{x\cdot x}$ (the $L_2$ norm), then $\|\cdot\|_\star=\|\cdot\|$.

Online convex optimization problems can be reduced to online \emph{linear} optimization problems in which the loss functions are constrained to be linear functions. The reduction follows by replacing the loss function $\ell_t(w)$ with the linear function $g_t\cdot w$, where $g_t$ is a subgradient of $\ell_t$ at $w_t$. Then, by definition, $g_t\cdot w_t-g_t\cdot u\ge \ell_t(w_t)-\ell_t(u)$. Therefore the regret of our algorithm with respect to  the linear loss functions $g_t\cdot w$ is an upper-bound on the regret with respect to the real loss functions $\ell_t$. Because of this reduction, many online convex optimization algorithms (including ours) are \emph{first order} algorithms, meaning they access the loss functions only through their subgradients. For the rest of this paper we will therefore assume that the losses are linear, $\ell_t(w)=g_t\cdot w$.

We will focus all of our lower bounds in Section \ref{sec:lowerbound} and algorithms in Section \ref{sec:algorithms} on the case in which the domain $W$ is an entire Hilbert space, so that $W$ has infinite diameter and no boundary. This case is very common in practical optimization optimization problems encountered in machine learning, in which any constraints are often only implicitly enforced via regularization. Our objective is to design lower bounds and algorithms such that $R_T(u)$ depends on $\|u\|$, $T$, and $\Lm$ without prior knowledge of these parameters.

In the following sections we use a compressed-sum notation where subscripts with colons indicate summations: $\sum_{t=1}^T g_t=g_{1:T}$, $ \sum_{t=1}^T \|g_t\|^2=\|g\|^2_{1:T}$, $\sum_{t=1}^T g_tw_t=(gw)_{1:T}$ and similarly for other indexed sums. Proofs are in the appendix when they do not immediately follow the result.

\section{A Frontier of Lower Bounds}\label{sec:lowerbound}

In this section we give our frontier of lower bounds for online optimization without prior information. First we describe our adversarial loss sequence and lower bound frontier along the $k$ dimension, and then we extend the argument to obtain the full two dimensional frontier parametrized by both $k$ and $\gamma$. 

\subsection{Trade-offs in the multiplicative constant $k$}
Given an algorithm, we establish a lower bound on its performance by constructing an adversarial sequence of subgradients $g_t\in \R$. This sequence sets $g_t=-1$ for $T-1$ iterations, where $T$ is chosen adversarially but can be made arbitrarily large, then sets $g_T=O(k\sqrt{T})$. Perhaps surprisingly, we prove that this simple strategy forces the algorithm to experience regret that is exponential in $\sqrt{T}/k$. We then express $\sqrt{T}/k$ as a constant multiple of $\tfrac{1}{k^2}L_t/L_{t-1}$, where $L_t=\max_{t'\le t}|g_t|$, capturing the algorithm's sensitivity to the big jump in the gradients between $T-1$ and $T$ in the adversarial sequence.

The cost that an algorithm pays when faced with the adversarial sequence is stated formally in the following Theorem.
\begin{Theorem}\label{thm:lowerboundkonly}
For any $k>0$, $T_0>0$, and any online optimization algorithm picking $w_t\in \R$, there exists a $T>T_0$, a $u\in \R$, and a fixed sequence $g_t\in \R$ on which the regret is:
\begin{align*}
    R_T(u)&=\sum_{t=1}^T g_tw_t-g_tu\\
    &\ge k\|u\|\Lm \log(T\|u\|+1)\sqrt{T} + \frac{\Lm}{T-1}\exp\left(\frac{\sqrt{T-1}}{8k}\right)\\
    &\ge k\|u\|\Lm \log(T\|u\|+1)\sqrt{T} + \max_{t\le T}\Lm\frac{L_{t-1}^2}{\|g\|^2_{1:t-1}}\exp\left[\frac{1}{2}\left(\frac{L_t/L_{t-1}}{288 k^2}\right)\right]
\end{align*}
where $L_t=\max_{t'\le t} \|g_{t'}\|$, and $\Lm=L_T=\max_{t\le T}\|g_t\|$.
\end{Theorem}

The first inequality in this bound demonstrates that it is impossible to guarantee sublinear regret without prior information while maintaining $O(\Lm\|u\|\log(\|u\|))$ dependence on $\Lm$ and $\|u\|$,\footnote{it \emph{is} possible to guarantee sublinear regret in exchange for $O(\Lm\|u\|^2)$ dependence, see \citet{orabona2016scale}} but the second inequality provides hope that if the loss sequence is limited to small jumps in $L_t$, then we might be able to obtain sublinear regret. Specifically, from the first inequality, observe that in order to bring the exponential term to lower than $O(T)$, the value of $k$ needs to be at least $\Omega(\sqrt{T}/\log(T))$, which causes the non-exponential term to become $O(T)$. However, the second inequality emphasizes that our high regret is the result of a large jump in the value of $L_t$, so that we might expect to do better if there are no such large jumps. Our upper bounds are given in the form of algorithms that guarantee regret matching the second inequality of this lower bound for any $k$, showing that we can indeed do well without prior information so long as $L_t$ does not increase too quickly.

\subsection{Trade-offs in the Logarithmic exponent $\gamma$}
To extend the frontier to the $\gamma$ dimension, we modify our adversarial sequence by setting $g_T = O(\gamma k^{1/\gamma} T^{1-1/2\gamma})$ instead of $O(k\sqrt{T})$. This results in a penalty that is exponential in $(\sqrt{T}/k)^{1/\gamma}$, which we express as a multiple of $(L_t/\gamma k^2L_{t-1})^{1/(2\gamma-1)}$. Since $\gamma\in(1/2,1]$, we are getting a larger exponential penalty even though the adversarial subgradients have decreased in size, illustrating that decreasing the logarithmic factor is very expensive.

The full frontier is stated formally in the following Theorem.

\begin{restatable}{Theorem}{lowerbound}\label{thm:lowerbound}
For any $\gamma\in(1/2,1]$, $k>0$, $T_0>0$, and any online optimization algorithm picking $w_t\in \R$, there exists a $T>T_0$, a $u\in \R$, and a sequence $g_1,\dots, g_T\in \R$ with $\|g_t\|\le \max(1,18\gamma(4k)^{1/\gamma}(t-1)^{1-1/2\gamma})$ on which the regret is:\footnote{The same result holds with in expectation for randomized algorithms with a deterministic sequence $g_t$.}
\begin{align*}
    R_T(u)&=\sum_{t=1}^T g_tw_t-g_tu\\
    &\ge k\|u\|\Lm \log^\gamma(T\|u\|+1)\sqrt{T} + \frac{\Lm}{T-1}\exp\left(\frac{(T-1)^{1/2\gamma}}{2(4k)^{1/\gamma}}\right)\\
    &\ge k\|u\|\Lm \log^\gamma(T\|u\|+1)\sqrt{T} + \max_{t\le T}\Lm\frac{L_{t-1}^2}{\|g\|^2_{1:t-1}}\exp\left[\frac{1}{2}\left(\frac{L_t/L_{t-1}}{288\gamma k^2}\right)^{1/(2\gamma-1)}\right]
\end{align*}

where $L_t=\max_{t'\le t} \|g_{t'}\|$ and $\Lm=L_T=\max_{t\le T}\|g_t\|$.
\end{restatable}

Again, the first inequality tells us that adversarial sequences can always deny the algorithm sublinear regret and the second inequality says that so long as $L_t$ grows slowly, we can still hope for sublinear regret. This time, however, the second inequality appears to blow up when $\gamma\to 1/2$. In this case, $\Lm=O(k^2)$ regardless of $T$ and so the value of $L_t/L_{t-1}$ is never very large, keeping the exponent in the second inequality less than 1 so that the singularity in the exponent does not send the bound to infinity.
This singularity at $\gamma=1/2$ tells us that the adversary does not need to be ``very adversarial'' in order to force us to experience exponential regret.

To gain some more intuition for what happens at $\gamma=1/2$, consider a model in which the adversary must commit ahead of time to some $\Lm$ (which corresponds to picking $k$), unknown to the optimization algorithm, such that $\|g_t\|\le \Lm$ for all $t$. When a bound $L_{\text{bound}}\ge \Lm$ is known to the algorithm ahead of time, then it is possible to achieve $O(\|u\| L_{\text{bound}}\sqrt{T\log(\|u\|T)})$ regret (e.g. see \citet{orabona2016coin}). However, note that when $\gamma=1/2$, committing to an appropriate $\Lm$ would not prevent an adversary from using the sequence of Theorem \ref{thm:lowerbound}. Therefore, Theorem \ref{thm:lowerbound} tells us that algorithms which achieve $O(\|u\| L_{\text{bound}}\sqrt{T\log(\|u\|T)})$ regret are inherently very fragile because if the bound is incorrect (which happens for large enough $k$), then the adversary can force the algorithm to suffer $\Lm\exp(O(T/\Lm))$ regret for arbitrarily large $T$.

Continuing with the model in which the adversary must commit to some unknown $\Lm$ ahead of time, suppose we are satisfied with $O(\|u\|\Lm\sqrt{T}\log^\gamma(\|u\|T))$ regret for some $\gamma>1/2$. In this case, after some (admittedly possibly very large) number of iterations, the exponential term in the second inequality no longer grows with $T$, and the adversarial strategy of Theorem \ref{thm:lowerbound} is not available because this strategy requires a choice of $\Lm$ that depends on $T$. Therefore an algorithm that guarantees regret matching the second inequality for some $k$ and $\gamma$ will obtain an asymptotic dependence on $T$ that is only $\log^\gamma(T)\sqrt{T}$.

These lower bounds show that there is a fundamental frontier of tradeoffs the between parameters $\gamma$ and $k$ and the exponential penalty. Now we proceed to derive algorithms that match any point on the frontier without prior information.

\section{Regret Analysis without Information}\label{sec:analysis}
In this section we provide the tools used to derive algorithms whose regret matches the lower bounds in the previous section.
Our algorithms make use of the Follow-the-Regularized-Leader (FTRL) framework, which is an elegant and intuitive way to design online learning algorithms (see \citet{shalev2011online,mcmahan2014survey} for detailed discussions). After seeing the $t^{th}$ loss of the online learning game, an FTRL algorithm chooses a function $\psi_t$ (called a \emph{regularizer}), and picks $w_{t+1}$ according to:
\begin{align*}
w_{t+1} &= \argmin_{w\in W} \psi_t(w) +\sum_{t'=1}^t\ell_{t'}(w)
\end{align*}

Careful choice of regularizers is obviously crucial to the success of such an algorithm, and in the following we provide simple conditions on $\psi$ sufficient for FTRL to achieve optimal regret without prior information. Our analysis generalizes many previous works for online learning with unconstrained $W$ (e.g. \citet{orabona2013dimension, orabona2014simultaneous,cutkosky2016online}) in which regret bounds were proved via arduous ad-hoc constructions. Further, our techniques improve the regret bound in the algorithm that does not require prior information of \citet{cutkosky2016online}. We note that an alternative set of conditions on regularizers was given in \citet{orabona2016coin} via an elegant reduction to coin-betting algorithms, but this prior analysis requires a known bound on $\Lm$.

Our regularizers $\psi_t$ take the form $\psi_t(w) = \frac{k}{a_t\eta_t}\psi(a_tw)$ for some fixed function $\psi$ and numbers $a_t$ and $\eta_t$. The value $k$ specifies the corresponding tradeoff parameter in the lower-bound frontier, while the function $\psi$ specifies the value of $\gamma$. The values for $a_t$ and $\eta_t$ do not depend on $k$ or $\psi$, but are carefully chosen functions of the observed gradients $g_1,\dots,g_t$ that guarantee the desired asymptotics in the regret bound.

%A key component of our analysis is a set of simple conditions on $\psi$ that enable us to prove a general regret bound on FTRL with regularizers $\psi_t$ (Theorem \ref{thm:parameterfreeregret}). We use these simple conditions to construct a family of functions $\psi$ for which application of our regret bound immediately implies regret matching the lower bound frontier.

%We consider two possible strategies for picking $a_t$ in our regret bound. One of these results in regret matching the form of the lower bound in Theorem \ref{thm:lowerbound}, while the other results in regret that saves a factor of $\log^\gamma(T)$ but has $R_T(0)=O(\sqrt{T})$. Both strategies result in optimal regret.

%We will give a carefully chosen formula for choosing $a_t$ and $\eta_t$, and
%Our algorithmic design strategy is to develop some conditions which ensure that a sequence of regularizers achieves optimal regret. We will call regularizers that satisfy these conditions ``adaptive regularizers'', because they allow the algorithm to adapt to unknown information. In Theorem \ref{thm:parameterfreeregret} we prove a bound on the regret of FTRL using adaptive regularizers, and in Section \ref{sec:algorithms} we provide some easy-to-check conditions that allow us to construct adaptive regularizers. By constructing an appropriate family of these adaptive regularizers, we obtain an family of algorithms that matches our lower bound frontier.

\subsection{Generalizing Strong Convexity}

Prior analyses of FTRL often make use of strongly-convex regularizers to simplify regret analysis, but it turns out that strongly-convex regularizers cannot match our lower bounds. Fortunately, there is a simple generalization of strong-convexity that will suffice for our purposes. This generalized notion is very similar to a dual version of the ``local smoothness" condition used in \citet{orabona2013dimension}. We define this generalization of strong-convexity below.

\begin{Definition}
Let $W$ be a convex space and
let $\sigma:W^2\to \R$ by an arbitrary function. We say a convex function $f:W\to \R$ is $\sigma(\cdot,\cdot)$-strongly convex with respect to a norm $\|\cdot\|$ if for all $x,y\in W$ and $g\in \partial f(x)$ we have
\[
f(y)\ge f(x) + g\cdot(y-x) + \frac{\sigma(x,y)}{2}\|x-y\|^2
\]
As a special case (and by abuse of notation), for any function $\sigma:W\to \R$ we define $\sigma(w,z)=\min(\sigma(w),\sigma(z))$ and define $\sigma(\cdot)$-strong convexity accordingly.
\end{Definition}

We'll usually just write $\sigma$-strongly convex instead of $\sigma(\cdot,\cdot)$-strongly convex since our definition is purely a generalization of the standard one. We will also primarily make use of the special case $\sigma(w,z)=\min(\sigma(w),\sigma(z))$.

\subsection{Adaptive regularizers}
Now we present a few definitions that will allow us to easily construct sequences of regularizers that achieve regret bounds without information. Intuitively, we require that our regularizers $\psi_t$ grow super-linearly in order to ensure that $\psi_t(w)+g_{1:t}w$ always has a minimal value. However, we do not want $\psi_t$ to grow quadratically because this will result in $O(\|u\|^2)$ regret. The formal requirements on the shape of $\psi_t$ are presented in the following definition:
\begin{Definition}
Let $W$ be a closed convex subset of a vector space such that $0\in W$. Any differentiable function $\psi:W\to \R$ that satisfies the following conditions:
\begin{enumerate}
\item $\psi(0)=0$.
\item $\psi(x)$ is $\sigma$-strongly-convex with respect to some norm $\|\cdot\|$ for some $\sigma:W\to \R$ such that $\|x\|\ge \|y\|$ implies $\sigma(x)\le \sigma(y)$.
\item For any $C$, there exists a $B$ such that $\psi(x)\sigma(x)\ge C$ for all $\|x\|\ge B$. 
\end{enumerate}
is called a \emph{$(\sigma,\|\cdot\|)$-adaptive regularizer}. We also define the useful auxiliary function $h(w) = \psi(w)\sigma(w)$ and by mild abuse of notation, we define $h^{-1}(x) = \max_{h(w)\le x}\|w\|$.
\end{Definition}

We will use adaptive regularizers as building blocks for our FTRL regularizers $\psi_t$, so it is important to have examples of such functions. We will provide some tools for finding adaptive regularizers in Section \ref{sec:algorithms}, but to keep an example in mind for now, we remark that $\psi(w)=(\|w\|+1)\log(\|w\|+1)-\|w\|$ is a $\left(\frac{1}{\|\cdot\|+1},\|\cdot\|\right)$-adaptive regularizer where $\|\cdot\|$ is the $L_2$ norm.

The following definition specifies the sequences $\eta_t$ and $a_t$ which we use to turn an adaptive regularizer into the regularizers used for our FTRL algorithms:
\begin{Definition}\label{dfn:regularizers}
Let $\|\cdot\|$ be a norm and $\|\cdot\|_\star$ be the dual norm ($\|x\|_\star = \sup_{\|y\|=1}x\cdot y$). Let $g_1,\dots,g_T$ be a sequence of subgradients and set $L_t=\max_{t'\le t} \|g_t\|_\star$.
Define the sequences $\frac{1}{\eta_t}$ and $a_t$ recursively by:
\begin{align*}
    \frac{1}{\eta_0^2}&=0\\
    \frac{1}{\eta_{t}^2} &= \max\left(\frac{1}{\eta_{t-1}^2}+2\|g_t\|_\star^2,L_t\|g_{1:t}\|_\star\right)\\
    a_1&=\frac{1}{(L_1\eta_1)^2}\\
    a_t&=\max\left(a_{t-1},\frac{1}{(L_t\eta_t)^2}\right)
\end{align*}

Suppose $\psi$ is a $(\sigma,\|\cdot\|)$-adaptive regularizer and $k>0$. Define
\begin{align*}
    \psi_t(w)&=\frac{k}{\eta_t a_t}\psi(a_t w)\\
    w_{t+1} &= \argmin_{w\in W} \psi_t(w) + g_{1:t}\cdot w\\
\end{align*}
\end{Definition}

Now without further ado, we give our regret bound for FTRL using these regularizers.

\begin{restatable}{Theorem}{parameterfreeregret}\label{thm:parameterfreeregret}
Suppose $\psi$ is a $(\sigma,\|\cdot\|)$-adaptive regularizer and $g_1,\dots,g_T$ is some arbitrary sequence of subgradients. Let $k\ge 1$, and let $\psi_t$ be defined as in Definition \ref{dfn:regularizers}.

Set
\begin{align*}
\sigmamin &=  \inf_{\|w\|\le h^{-1}\left(10/k^2\right)} k\sigma(w)\\
D &= \max_{t} \frac{L_{t-1}^2}{(\|g\|_\star^2)_{1:t-1}}h^{-1}\left(\frac{5L_t}{k^2L_{t-1}}\right)\\
Q_T&=2\frac{\|g\|_{1:T}}{\Lm}
\end{align*}

Then FTRL with regularizers $\psi_t$ achieves regret
\begin{align*}
    R_T(u) & \le \frac{k}{Q_T\eta_T}\psi(Q_Tu) +\frac{45\Lm}{\sigmamin}+2\Lm D\\
    &\le k\Lm\frac{\psi(2uT)}{\sqrt{2T}} +\frac{45\Lm}{\sigmamin}+2\Lm D
\end{align*}
\end{restatable}

This bound consists of three terms, the first of which will correspond to the $\sqrt{T}$ term in our lower bounds and the last of which will correspond to the exponential penalty. The middle term is a constant independent of $u$ and $T$. To unpack a specific instantiation of this bound, consider the example adaptive regularizer $\psi(w)=(\|w\|+1)\log(\|w\|+1)-\|w\|$. For this choice of $\psi$, we have $\psi(2uT)/\sqrt{2T}=O(\|u\|\sqrt{T}\log(T\|u\|+1))$ so that the first term in the regret bound matches the $\sqrt{T}$ term in our lower bound with $\gamma=1$. Roughly speaking, $h(w)\approx \log(w)$, so that $h^{-1}(x)\approx \exp(x)$ and the quantity $D=\max_t\frac{L_{t-1}^2}{(\|g\|_\star^2)_{1:t-1}}h^{-1}\left(\frac{5L_t}{k^2L_{t-1}}\right)$ matches the exponential penalty in our lower bound. In the following section we formalize this argument and exhibit a family of adaptive regularizers that enable us to design algorithms whose regret matches any desired point on the lower bound frontier.

\section{Optimal Algorithms}\label{sec:algorithms}

In this section we construct specific adaptive regularizers in order to obtain optimal algorithms using our regret upper bound of Theorem \ref{thm:parameterfreeregret}. The results in the previous section hold for arbitrary norms, but from this point on we will focus on the $L_2$ norm.
Our regret upper bound expresses regret in terms of the function $h^{-1}$. Inspection of the bound shows that if $h^{-1}(x)$ is exponential in $x^{1/(2\gamma-1)}$, and $\psi(w)=O(\|w\|\log^\gamma(\|w\|+1))$, then our upper bound will match (the second inequality in) our lower bound frontier. The following Collary formalizes this observation.

\begin{Corollary}\label{thm:orderbound}
If $\psi$ is an  $(\sigma,\|\cdot\|)$-adaptive regularizer such that
\begin{align*}
\psi(x)\sigma(x) &\ge  \Omega(\gamma\log^{2\gamma-1}(\|x\|))\\
\psi(x)&\le O(\|x\|\log^\gamma(\|x\|+1))
\end{align*}
then for any $k\ge 1$, FTRL with regularizers $\psi_t(w) = \frac{k}{a_t\eta_t}\psi(a_tw)$ yields regret
\begin{align*}
R_T(u) &\le O
\left[
 k\Lm\sqrt{T}\|u\|\log^\gamma(T\|u\|+1) +
 \max_t\frac{\Lm L_{t-1}^2}{\|g\|^2_{1:t-1}}\exp
 \left[ 
   O\left(\left(
    \frac{L_t}{k^2\gamma L_{t-1}}
   \right)^{1/(2\gamma-1)}\right)
 \right]
\right]
\end{align*}
We call regularizers that satisfy these conditions $\gamma$-optimal.
\end{Corollary}

With this Corollary in hand, to match our lower bound frontier we need only construct a $\gamma$-optimal adaptive regularizer for all $\gamma\in(1/2,1]$. Constructing adaptive regularizers is made much simpler with Proposition \ref{thm:onedimconditions} below. This proposition allows us to design adaptive regularizers in high dimensional spaces by finding simple one-dimensional functions. It can be viewed as taking the place of arguments in prior work \citep{mcmahan2014unconstrained,orabona2016coin,cutkosky2016online} that reduce high dimensional problems to one-dimensional problems by identifying a ``worst-case" direction for each subgradient $g_t$.

\begin{restatable}{Proposition}{onedimconditions}\label{thm:onedimconditions}
Let $\|\cdot\|$ be the $L_2$ norm $(\|w\|=\|w\|_2=\sqrt{w\cdot w})$. Let $\phi$ be a three-times differentiable function from the non-negative reals to the reals that satisfies
\begin{enumerate}
\item $\phi(0)=0$.
\item $\phi'(x)\ge 0$.
\item $\phi''(x)\ge 0$.
\item $\phi'''(x)\le 0$.
\item $\lim_{x\to\infty} \phi(x)\phi''(x)=\infty$.
\end{enumerate}
Then $\psi(w)=\phi(\|w\|)$ is a $(\phi''(\|\cdot\|),\|\cdot\|)$-adaptive regularizer.
\end{restatable}

Now we are finally ready to derive our first optimal regularizer:
\begin{Proposition}\label{thm:rexregularizer}
Let $\|\cdot\|$ be the $L_2$ norm. Let
$\phi(x) = (x+1)\log(x+1)-x$. Then $\psi(w)=\phi(\|w\|)$ is a $1$-optimal, $(\phi''(\|\cdot\|),\|\cdot\|)$-adaptive regularizer.
\end{Proposition}
\begin{proof}
We can use Proposition \ref{thm:onedimconditions} to prove this with a few simple calculations:
\begin{align*}
\phi(0)&=0\\
\phi'(x) &= \log(x+1)\\
\phi''(x) &= \frac{1}{x+1}\\
\phi'''(x) &= -\frac{1}{(x+1)^2}\\
\phi(x)\phi''(x) &= (\log(x+1)-\frac{x}{x+1})
\end{align*}
Now the conclusion of the Proposition is immediate from Proposition \ref{thm:onedimconditions} and inspection of the above equations.
\end{proof}

A simple application of Corollary \ref{thm:orderbound} shows that FTRL with regularizers $\psi_t(w)=\frac{k}{\eta_t}((\|w\|+1)\log(\|w\|+1)-\|w\|)$ matches our lower bound with $\gamma=1$ for any desired $k$.

In fact, the result of Proposition \ref{thm:rexregularizer} is a more general phenomenon:
\begin{Proposition}\label{thm:matchthebound}
Let $\|\cdot\|$ be the $L_2$ norm. Given $\gamma\in(1/2,1]$, set $\phi(x) = \int_0^x \log^\gamma(z+1)\ dz$. Then $\psi(w)=\phi(\|w\|)$ is a $\gamma$-optimal, $(\phi''(\|\cdot\|),\|\cdot\|)$-adaptive regularizer.
\end{Proposition}
\begin{proof}
\begin{align*}
    \phi(0) &= 0\\
    \phi'(x) &= \log^\gamma(x+1)\\
    \phi''(x) &= \gamma\frac{\log^{\gamma-1}(x+1)}{x+1}\\
    \phi'''(x) &= \gamma(\gamma-1)\frac{\log^{\gamma-2}(x+1)}{(x+1)^2}-\gamma\frac{\log^{\gamma-1}(x+1)}{(x+1)^2}
\end{align*}
Since $\gamma\le 1$, $\phi'''(x)\le 0$ and so $\phi$ satisfies the first four conditions of Proposition \ref{thm:onedimconditions}. It remains to characterize $\phi(x)$ and $\phi(x)\phi''(x)$, which we do by finding lower and upper bounds on $\phi(x)$:

For a lower bound, we have
\begin{align*}
    \frac{1}{2}\frac{d}{dx}x\log^\gamma(x+1)&= \frac{1}{2}\left(\log^\gamma(x+1)+\gamma\frac{x}{x+1}\log^{\gamma-1}(x+1)\right)\\
    &\le \log^\gamma(x+1)
\end{align*}
where the inequality follows since $\frac{x}{x+1}\le \log(x+1)$, which can be verified by differentiating both sides.
Therefore $\phi(x) \ge \frac{1}{2}x\log^\gamma(x+1)$. This lower-bound implies
\begin{align*}
    \phi(x)\phi''(x) &\ge \frac{1}{2}\gamma\frac{x}{x+1} \log^{2\gamma-1}(x+1)
\end{align*}
which gives us the last condition in Proposition \ref{thm:onedimconditions}, as well as the first condition for $\gamma$-optimality.

Similarly, we have
\begin{align*}
    \frac{d}{dx}x\log^\gamma(x+1)&=\left(\log^\gamma(x+1)+\gamma\frac{x}{x+1}\log^{\gamma-1}(x+1)\right)\\
    &\ge \log^\gamma(x+1)
\end{align*}
This implies $\phi(x)\le x\log(x+1)$ which gives us the second condition for $\gamma$-optimality.
\end{proof}

Thus, by applying Theorem \ref{thm:parameterfreeregret} to the regularizers of Proposition \ref{thm:matchthebound}, we have a family of algorithms that matches our family of lower-bounds up to constants. The updates for these regularizers are extremely simple:
\begin{align*}
    w_{t+1} = 
    -\frac{g_{1:t}}{a_t\|g_{1:t}\|}\left[\exp\left((\eta_t\|g_{1:t}\|/k)^{1/\gamma}\right)-1\right]
\end{align*}

The guarantees of Theorem \ref{thm:parameterfreeregret} do not make any assumptions on how $k$ is chosen, so that we could choose $k$ using prior knowledge if it is available. For example, if a bound on $L_t/L_{t-1}$ is known, we can set $k\ge \sqrt{\max_t L_t/L_{t-1}}$. This reduces the exponentiated quantity $\max_t L_t/k^2L_{t-1}$ to a constant, leaving a regret of $O(\|u\|\log(T\|u\|+1)\Lm\sqrt{T\max_t L_t/L_{t-1}})$. This bound holds without requiring a bound on $\Lm$. Thus our algorithms open up an intermediary realm in which we have no bounds on $\|u\|$ or $\Lm$, and yet we can leverage some other information to avoid the exponential penalty.

\section{\freerex}\label{sec:freerex}
Now we explicitly describe an algorithm, along with a fully worked-out regret bound. The norm $\|\cdot\|$ used in the following is the $L_2$ norm ($\|w\|=\sqrt{w\cdot w}$), and our algorithm uses the adaptive regularizer $\psi(w)=(\|w\|+1)\log(\|w\|+1)-\|w\|$. Similar calculations could be performed for arbitrary $\gamma$ using the regularizers of Proposition \ref{thm:matchthebound}, but we focus on the $\gamma=1$ because it allows for simpler and tighter analysis through our closed-form expression for $\psi$. Since we do not require any information about the losses, we call our algorithm \freerex\ for Information-\textbf{free} \textbf{R}egret via \textbf{ex}ponential updates.

\begin{algorithm}
   \caption{\freerex}
   \label{alg:freerex}
\begin{algorithmic}
   \STATE {\bfseries Input:} $k$.
   \STATE {\bfseries Initialize:} $\frac{1}{\eta_0^2}\gets 0$, $a_0\gets 0$, $w_1\gets 0$, $L_0\gets 0$, $\psi(w)=(\|w\|+1)\log(\|w\|+1)-\|w\|$.
   \FOR{$t=1$ {\bfseries to} $T$}
   \STATE Play $w_t$, receive subgradient $g_t\in \partial \ell_t(w_t)$.
   \STATE $L_t\gets \max(L_{t-1},\|g_t\|)$.
   \STATE $\frac{1}{\eta_t^2}\gets\max\left(\frac{1}{\eta_{t-1}^2}+2\|g_t\|^2, L_t \|g_{1:t}\|\right)$.
   \STATE $a_t\gets\max(a_{t-1},1/(L_t\eta_t)^2)$.
   \STATE //Set $w_{t+1}$ using FTRL update
   \STATE $w_{t+1} \gets -\frac{g_{1:t}}{a_t\|g_{1:t}\|}\left[\exp\left(\frac{\eta_t\|g_{1:t}\|}{k}\right)-1\right]$ // $=\argmin_w\left[\frac{k\psi(a_tw)}{a_t\eta_t}+g_{1:t}w\right]$
   \ENDFOR
\end{algorithmic}
\end{algorithm}

\begin{Theorem}\label{thm:freerexregret}
The regret of \freerex\ (Algorithm \ref{alg:freerex}) is bounded by
\begin{align*}
      R_T(u) &\le k\|u\|\sqrt{2\|g\|^2_{1:T}+\Lm\max_{t\le T}\|g_{1:t}\|} \log\left(\frac{2\|g\|_{1:T}}{\Lm}\|u\|+1\right)+\frac{45\Lm}{k}\exp(10/k^2+1)\\
    &\quad\quad+ 2\Lm\max_t \frac{L_{t-1}^2}{\|g\|^2_{1:t-1}}\left[\exp\left(\frac{5L_t}{k^2L_{t-1}}+1\right)-1\right]
\end{align*}
\end{Theorem}

\begin{proof}
Define $\phi(x)=(x+1)\log(x+1)-x$. Then $\psi(w)=(\|w\|+1)\log(\|w\|+1)-\|w\|$ is a $(\phi''(\|\cdot\|),\|\cdot\|)$-adaptive regularizer by Proposition \ref{thm:rexregularizer}. Therefore we can immediately apply Theorem \ref{thm:parameterfreeregret} to obtain
\begin{align*}
    R_T(u) &\le \frac{k}{Q_T\eta_T}\psi(Q_Tu)+\frac{45\Lm}{\phi''_{\text{min}}}+ 2\Lm D
\end{align*}
where we've defined $\phi''_{\text{min}}=\inf_{\|w\|\le h^{-1}(10/k^2)}k\phi''(\|w\|)$.

We can compute (for non-negative $x$):
\begin{align*}
    \phi(x) &\le (x+1)\log(x+1)\\
    \phi''(x) &=\frac{1}{x+1}\\
    h(w)&=\phi(\|w\|)\phi''(\|w\|)=\left(\log(\|w\|+1)-\frac{\|w\|}{\|w\|+1}\right)\\
    &\ge \log(\|w\|+1)-1
\end{align*}

From Proposition \ref{thm:etarates} (part 2) we have $\frac{1}{\eta_T}\le \sqrt{2\|g\|^2_{1:T} + \Lm\max_{t\le T}\|g_{1:t}\|}$. We also have $(\|w\|+1)\log(\|w|+1)-\|w\|=\|w\|\log(\|w\|+1)+\log(\|w\|+1)-\|w\|\le \|w\|\log(\|w\|+1)$, so we are left with
\begin{align*}
    R_T(u) &\le \frac{k}{\eta_T} \|u\|\log(Q_T\|u\|+1)+\sup_{\|w\|\le h^{-1}(\frac{10}{k^2})}\frac{45(\|w\|+1)}{k} + 2\Lm D\\
    &=k\sqrt{2\|g\|^2_{1:T}+\Lm\max_{t\le T}\|g_{1:t}\|}\|u\|\log(a_T\|u\|)+1)+\frac{45\Lm}{k}\left[h^{-1}\left(\frac{10}{k^2}\right)+1\right] \\
    &\quad+ 2\Lm D
\end{align*}
Now it remains to bound $h^{-1}(10/k^2)$ and $D$. From our expression for $h$, we have
\begin{align*}
    h^{-1}(x/k^2) &\le \exp\left[\frac{x}{k^2}+1\right]-1
\end{align*}

Therefore we have
\begin{align*}
    h^{-1}(10/k^2)&\le \exp(10/k^2+1)-1\\
    D &= 2\max_t \frac{L_{t-1}^2}{(\|g\|_\star^2)_{1:t-1}}h^{-1}\left(\frac{5L_t}{k^2L_{t-1}}\right)\\
    &\le 2\max_t \frac{L_{t-1}^2}{(\|g\|_\star^2)_{1:t-1}}\left[\exp\left(\frac{5L_t}{k^2L_{t-1}}+1\right)-1\right]
\end{align*}

Substituting the value $Q_T=2\frac{\|g\|_{1:T}}{\Lm}$, we conclude 
\begin{align*}
    R_T(u) &\le k\sqrt{2\|g\|^2_{1:T}+\Lm\max_{t\le T}\|g_{1:t}\|} \|u\|\log\left(\frac{2\|g\|_{1:T}}{\Lm}\|u\|+1\right) \\
    &\quad\quad+\frac{45\Lm}{k}\exp(10/k^2+1)+ 2\Lm D
\end{align*}
From which the result follows by substituting in our expression for $D$.

\end{proof}

As a specific example, for $k=\sqrt{5}$ we numerically evaluate the bound to get
\begin{align*}
      R_T(u) &\le \|u\|\sqrt{10\|g\|^2_{1:T}+5\Lm\max_{t\le T}\|g_{1:t}\|} \log\left(\frac{2\|g\|_{1:T}}{\Lm}\|u\|+1\right)+405\Lm\\
    &\quad\quad+ 2\Lm\max_t \frac{L_{t-1}^2}{\|g\|^2_{1:t-1}}\left[\exp\left(\frac{L_t}{L_{t-1}}+1\right)-1\right]
\end{align*}

\section{Conclusions}

We have presented a frontier of lower bounds on the worst-case regret of any online convex optimization algorithm without prior information. This frontier demonstrates a fundamental trade-off at work between $ku\Lm\log^\gamma(Tu+1)$ and $\exp\left[\left(\max_t \frac{L_t}{\gamma k^2L_{t-1}}\right)^{\frac{1}{2\gamma-1}}\right]$ terms. We also present some easy-to-use theorems that allow us to construct algorithms that match our lower bound for any chosen $k$ and $\gamma$. Note that by virtue of not requiring prior information, our algorithms are nearly hyperparameter-free. They only require the essentially unavoidable trade-off parameters $k$ and $\gamma$. Since our analysis does not make assumptions about the loss functions or comparison point $u$, the parameters $k$ and $\gamma$ can be freely chosen by the user. Unlike other algorithms that require $\|u\|$ or $\Lm$, there are no unknown constraints on these parameters.

Our results also open a new perspective on optimization algorithms by casting using prior information as a tool to avoid the exponential penalty. Previous algorithms that require bounds on the diameter of $W$ or $\Lm$ can be viewed as addressing this issue. We show that it also possible to avoid the exponential penalty by using a known bound on $\max_t L_t/L_{t-1}$, leading to a regret of $\tilde O(\|u\|\Lm\sqrt{T\max_t L_t/L_{t-1}})$.

%We also show that, essentially by virtue of being adaptive, our algorithms automatically achieve very small regret on smooth losses so long as the comparison point has small total loss.

Although we answer some important questions, there is still much to do in online learning without prior information. For example, it is possible to obtain $O(\|u\|^2\Lm\sqrt{T})$ regret without prior information \citep{orabona2016scale}, so it should be possible to extend our lower-bound frontier beyond $\|u\|\log(\|u\|)$. Further, it would be valuable to further characterize the conditions for which the adversary can guarantee regret that is exponential in $T$. We showed that one such condition is that there must be a large jump in the value of $L_t$, but there may very well be others. Fully characterizing these conditions should allow us design algorithms that smoothly interpolate between ``nice'' environments that do not satisfy the conditions and fully adversarial ones that do. 

Finally, while our analysis allows for the use of arbitrary norms, we focus our examples on the $L_2$ norm. It may be interesting to design adaptive regularizers with respect to a more diverse set of norms, or to extend our theory to encompass time-changing norms.

\small
\bibliography{all}

\appendix

\section{Lower Bound Proof}
Before getting started, we need one technical observation:
\begin{Proposition}\label{thm:rtbound}
Let $k>0$, $\gamma\in(1/2,1]$. Set 
\[
Z_t = \frac{t^{1-1/2\gamma}}{2t}\left[\exp\left(\frac{t^{1/2\gamma}}{(4k)^{1/\gamma}}\right)-1\right]
\]
and set $r_t= Z_t-Z_{t-1}$. Then for all sufficiently large $T$,
\begin{align*}
    r_T\ge \frac{Z_{T-1}}{3\gamma(4k)^{1/\gamma}(T-1)^{1-1/2\gamma}}
\end{align*}
\end{Proposition}
\begin{proof}
We have
\begin{align*}
    \left.\frac{d}{dt}\right|_{t=T}Z_t&= \frac{1}{4\gamma(4k)^{1/\gamma}T}\exp\left(\frac{T^{1/2\gamma}}{(4k)^{1/\gamma}}\right)+\frac{1}{4\gamma}T^{-1-1/2\gamma}-\frac{1}{4\gamma}T^{-1-1/2\gamma}\exp\left(\frac{T^{1/2\gamma}}{(4k)^{1/\gamma}}\right)
\end{align*}
For sufficiently large $T$, this quantity is positive and increasing in $T$. Therefore for sufficiently large $T$,
\begin{align*}
    r_T&\ge \left.\frac{d}{dt}\right|_{t=T-1}Z_t\\
    &= \frac{1}{4\gamma(4k)^{1/\gamma}(T-1)}\exp\left(\frac{(T-1)^{1/2\gamma}}{(4k)^{1/\gamma}}\right)+\frac{1}{4\gamma}(T-1)^{-1-1/2\gamma}-\frac{1}{4\gamma}(T-1)^{-1-1/2\gamma}\exp\left(\frac{(T-1)^{1/2\gamma}}{(4k)^{1/\gamma}}\right)\\
    &\ge \frac{1}{5\gamma(4k)^{1/\gamma}(T-1)}\exp\left(\frac{(T-1)^{1/2\gamma}}{(4k)^{1/\gamma}}\right)\\
    &= \frac{2}{5\gamma(4k)^{1/\gamma}(T-1)^{1-1/2\gamma}}\left(Z_{T-1}+\frac{(T-1)^{1-1/2\gamma}}{2(T-1)}\right)\\
    &\ge \frac{1}{3\gamma(4k)^{1/\gamma}(T-1)^{1-1/2\gamma}}Z_{T-1}
\end{align*}
where the third inequality holds only for sufficiently large $T$.
\end{proof}

Now we prove Theorem \ref{thm:lowerbound}, restated below. Theorem \ref{thm:lowerboundkonly} is an immediate consequence of Theorem \ref{thm:lowerbound}, so we do not prove it seperately.
\lowerbound*

\begin{proof}
We prove the Theorem for randomized algorithms and expected regret, as this does not overly complicate the argument.
Our proof technique is very similar to that of \citep{cutkosky2016online}, but we use more careful analysis to improve the bound. Intuitively, the adversarial sequence foils the learner by repeatedly presenting it with the subgradient $g_t=-1$ until the learner's expected prediction $\E[w_t]$ crosses some threshold. If $\E[w_t]$ does not increase fast enough to pass the threshold, then we show that there is some large $u\gg 1$ for which $R_T(u)$ exceeds our bound. However, if $\E[w_t]$ crosses this threshold, then the adversary presents a large positive gradient which forces the learner to have a large $R_T(0)$.

Define $\hat w_t=\E[w_t|g_{t'}=-1\text{ for all }t'<t]$. Without loss of generality, assume $\hat w_1=0$. Note that $\hat w_t$ can be computed by an adversary without access to the algorithm's internal randomness.

Let $S_n=\sum_{t=1}^n \hat w_t$. Let $Z_t=\frac{t^{1-1/2\gamma}}{2t}\left[\exp\left(\frac{t^{1/2\gamma}}{(4k)^{1/\gamma}}\right)-1\right]$, and set $r_t=Z_t-Z_{t-1}$ Suppose $T_1>T_0$ is such that
\begin{enumerate}
\item For all $t_1>t_2>T_1$, $Z_{t_1}>Z_{t_2}$.
\item For all $t>T_1$, $r_t\ge \frac{Z_{t-1}}{3\gamma(4k)^{1/\gamma}(t-1)^{1-1/2\gamma}}$ (by Proposition \ref{thm:rtbound}).
\item For all $t>T_1$,
\begin{align*}
    \frac{1}{4}\left[\exp\left(\frac{t^{1/2\gamma}}{(4k)^{1/\gamma}}\right)-1\right]&\ge \frac{1}{t-1}\exp\left(\frac{(t-1)^{1/2\gamma}}{2(4k)^{1/\gamma}}\right)
\end{align*}
\item for all $t>T_1$,
\begin{align*}
    \frac{1}{36\gamma(4k)^{1/\gamma}(t-1)}\left[\exp\left(\frac{(t-1)^{1/2\gamma}}{(4k)^{1/\gamma}}\right)-1\right]\ge \frac{1}{(t-1)}\exp\left(\frac{(t-1)^{1/2\gamma}}{2(4k)^{1/\gamma}}\right)
\end{align*}
\item For all $t>T_1$,
\begin{align*}
    \frac{1}{t-1}\exp\left(\frac{(t-1)^{1/2\gamma}}{2(4k)^{1/\gamma}}\right)&\ge\exp\left[\frac{1}{4}\left(\frac{1}{288 \gamma k^2}\right)^{1/(2\gamma-1)}\right]\\
\end{align*}
\item For all $t>T_1$,
\begin{align*}
18\gamma(4k)^{1/\gamma}(T-1)^{1-1/2\gamma}\ge 1
\end{align*}
\end{enumerate}

We consider the quantity $\liminf_{n\to\infty} \frac{S_n}{Z_n}$. There are two cases, either the $\liminf$ is less than 1, or it is not.

\noindent{\bf Case 1: $\liminf_{n\to\infty} \frac{S_n}{Z_n}<1$}

In this case, there must be some $T>T_1$ such that $S_T<Z_T$. We use the adversarial strategy of simply giving $g_t=-1$ for all $t\le T$. Because of this, $\E[w_t|g_1,\dots,g_{t-1}]=\hat w_t$ so that
\begin{align*}
    \E[R_T(u)] &= \sum_{t=1}^T g_t\E[w_t|g_1,\dots,g_{t-1}]-g_tu\\
    &=\sum_{t=1}^T g_t\hat w_t-g_tu\\
    &= Tu-S_T\\
    &\ge Tu-\frac{T^{1-\frac{1}{2\gamma}}}{2T}\left[\exp\left(\frac{T^{1/2\gamma}}{(4k)^{1/\gamma}}\right)-1\right]\\
    &\ge Tu-\frac{1}{2}\left[\exp\left(\frac{T^{1/2\gamma}}{(4k)^{1/\gamma}}\right)-1\right]
\end{align*}
Set $u=\frac{1}{T}\left[\exp\left(\frac{T^{1/2\gamma}}{(4k)^{1/\gamma}}\right)-1\right]$. Then clearly 
\begin{align*}
    \E[R_T(u)]&\ge Tu-\frac{1}{2}\left[\exp\left(\frac{T^{1/2\gamma}}{(4k)^{1/\gamma}}\right)-1\right]\\
    &\ge \frac{1}{2}Tu\\
    &=\frac{1}{4}Tu+\frac{1}{4}\left[\exp\left(\frac{T^{1/2\gamma}}{(4k)^{1/\gamma}}\right)-1\right]
\end{align*}
Now observe that we have chosen $u$ carefully so that
\begin{align*}
    \sqrt{T} = 4k\log^\gamma(Tu+1)
\end{align*}
Therefore we can write
\begin{align*}
    \E[R_T(u)]&\ge \frac{1}{4}Tu+\frac{1}{4}\left[\exp\left(\frac{T^{1/2\gamma}}{(4k)^{1/\gamma}}\right)-1\right]\\
    &=k\|u\|\log^\gamma(T\|u\|+1)\sqrt{T} + \frac{1}{4}\left[\exp\left(\frac{T^{1/2\gamma}}{(4k)^{1/\gamma}}\right)-1\right]\\
    &=k\|u\|\Lm\log^\gamma(T\|u\|+1)\sqrt{T} + \frac{\Lm}{4}\left[\exp\left(\frac{T^{1/2\gamma}}{(4k)^{1/\gamma}}\right)-1\right]
\end{align*}
where we have used $\Lm=1$ to insert factors of $\Lm$ where appropriate.

Observing that $L_t/L_{t-1}=1$ for all $t$, we can also easily conclude (using properties 3 and 5 of $T_1$):
\begin{align*}
    \E[R_T(u)]&\ge k\|u\|\Lm \log^\gamma(T\|u\|+1)\sqrt{T} + \frac{\Lm}{T-1}\exp\left(\frac{(T-1)^{1/2\gamma}}{2(4k)^{1/\gamma}}\right)\\
    &\ge k\|u\|\Lm \log^\gamma(T\|u\|+1)\sqrt{T} + \max_{t\le T}\Lm\frac{L_{t-1}^2}{\sum_{t'=1}^{t-1}\|g_{t'}\|^2}\exp\left[\frac{1}{2}\left(\frac{L_t/L_{t-1}}{288\gamma k^2}\right)^{1/(2\gamma-1)}\right]
\end{align*}
\noindent{\bf Case 2: $\liminf_{n\to\infty} \frac{S_n}{Z_n}\ge 1$}

By definition of $\liminf$, there exists some $T_2>T_1$ and $Q\ge 1$ such that $S_{T_2}\le \frac{3}{2}Q Z_{T_2}$ and for all $t>T_2$, $S_t > \frac{3Q}{4} Z_t$. 

%for any decreasing sequence of positive numbers $\epsilon_i$ with $\epsilon_i\to 0$ as $i\to\infty$ and $\epsilon_i<Q/2$, there is a sequence of increasing indices $T_1,T_2,\dots$ with $T_0\le T_1$ such that $S_{T_i}\le (Q+\epsilon_i) Z_{T_i}$. and $S_t> \frac{3Q}{4} Z_{t}$ for all $t\ge T_1$.

Suppose for contradiction that $\hat w_t\le \frac{Q}{2}r_t$ for all $t>T_2$. Then for all $T>T_2$,
\begin{align*}
    S_T &= S_{T_2}+\sum_{t=T_2+1}^T \hat w_t\\
    &\le \frac{3}{2}QZ_{T_2}+\frac{Q}{2}Z_T-\frac{Q}{2}Z_{T_2}\\
    &= \frac{Q}{2}Z_T +QZ_{T_2}
\end{align*}
Since the second term does not depend on $T$, this implies that for sufficiently large $T$, $\frac{S_T}{Z_T}\le \frac{3}{4}QZ_T$, which contradicts our choice of $T_2$. Therefore $\hat w_t>\frac{Q}{2}r_t$ for some $t>T_2$. 

Let $T$ be the the smallest index $T>T_2$ such that $\hat w_T> \frac{Q}{2}r_T$. Since $\hat w_t\le \frac{Q}{2}r_t$ for $t<T$, we have
\begin{align*}
    S_{T-1} &\le \frac{Q}{2}Z_{T-1} +Q Z_{T_2}\le 2QZ_{T-1}
\end{align*}
where we have used property 1 of $T_1$ to conclude $Z_{T_2}\le Z_{T-1}$.

Our adversarial strategy is to give $g_t=-1$ for $t<T$, then $g_T = 18\gamma(4k)^{1/\gamma}(T-1)^{1-1/2\gamma}$. We evaluate the regret at $u=0$ and iteration $T$. Since $g_t=-1$ for $t<T$, $\E[w_t|g_1,\dots,g_{t-1}]=\hat w_t$ for $t\le T$ and so
\begin{align*}
    \E[R_T(u)]&=-S_{T-1}+g_Tw_T\\
    &\ge g_T\frac{Q}{2}r_T -2QZ_{T-1}\\
    &\ge \frac{Q}{2}\frac{18\gamma(4k)^{1/\gamma}(T-1)^{1-1/2\gamma}}{3\gamma(4k)^{1/\gamma}(T-1)^{1-1/2\gamma}}Z_{T-1}-2QZ_{T-1}\\
    &=QZ_{T-1}\\
    &\ge Z_{T-1}
\end{align*}
where we have used $Q\ge 1$ in the last line.
Now we use the fact that $\Lm = 18\gamma (4k)^{1/\gamma}(T-1)^{1-1/2\gamma}$ (by property 6 of $T_1$) to write
\begin{align*}
\E[R_T(u)]&\ge Z_{T-1}\\
    &=\frac{1}{18\gamma(4k)^{1/\gamma})}\frac{\Lm}{T-1}\left[\exp\left(\frac{(T-1)^{1/2\gamma}}{(4k)^{1/\gamma}}\right)-1\right]\\
    &\ge \frac{\Lm}{T-1}\exp\left(\frac{(T-1)^{1/2\gamma}}{2(4k)^{1/\gamma}}\right)
\end{align*}
where we have used the fourth assumption on $T_1$ in the last line.

Since we are considering $u=0$, we can always insert arbitrary multiples of $u$:
\begin{align*}
    \E[R_T(u)]&\ge \frac{\Lm}{T-1}\exp\left(\frac{(T-1)^{1/2\gamma}}{2(4k)^{1/\gamma}}\right)\\
    &=k\|u\|\Lm \log^\gamma(T\|u\|+1)\sqrt{T} + \frac{\Lm}{T-1}\exp\left(\frac{(T-1)^{1/2\gamma}}{2(4k)^{1/\gamma}}\right)
\end{align*}

Now we relate the quantity in the exponent to $L_t/L_{t-1}$. We have $L_T = g_T$ and $L_{T-1}=1$ so that
\begin{align*}
    L_T/L_{T-1} = 18\gamma(4k)^{1/\gamma}(T-1)^{1-1/2\gamma}
\end{align*}
Therefore
\begin{align*}
    (T-1)^{1/2\gamma}&=\left(\frac{L_T/L_{T-1}}{18\gamma(4k)^{1/\gamma}}\right)^{\frac{1}{2\gamma(1-1/2\gamma)}}\\
    &=\left(\frac{L_T/L_{T-1}}{18\gamma(4k)^{1/\gamma}}\right)^{1/(2\gamma-1)}\\
    \frac{(T-1)^{1/2\gamma}}{(4k)^{1/\gamma}}&=\left(\frac{L_T/L_{T-1}}{18\gamma(4k)^2}\right)^{1/(2\gamma-1)}\\
    &=\left(\frac{L_T/L_{T-1}}{288 \gamma k^2}\right)^{1/(2\gamma-1)}
\end{align*}

Now observe that $\frac{1}{T-1} = \frac{L_{T-1}^2}{\sum_{t=1}^{T-1}\|g_t\|^2}$ so that we have
\begin{align*}
    \frac{\Lm}{T-1}\exp\left(\frac{(T-1)^{1/2\gamma}}{2(4k)^{1/\gamma}}\right)&=\Lm\frac{L_{T-1}^2}{\sum_{t=1}^{T-1}\|g_t\|^2}\exp\left[\frac{1}{2}\left(\frac{L_T/L_{T-1}}{288 \gamma k^2}\right)^{1/(2\gamma-1)}\right]
\end{align*}

Further, since $\frac{1}{t-1} =\frac{L_{t-1}^2}{\sum_{t'=1}^{t-1}\|g_{t'}\|^2}$ for all $t\le T$, condition 5 on $T_1$ tells us that
\begin{align*}
    \frac{\Lm}{T-1}\exp\left(\frac{(T-1)^{1/2\gamma}}{2(4k)^{1/\gamma}}\right)&\ge\Lm\exp\left[\frac{1}{2}\left(\frac{1}{288\gamma k^2}\right)^{1/(2\gamma-1)}\right]\\
    &= \max_{t\le T-1}\Lm\frac{L_{t-1}^2}{\sum_{t'=1}^{t-1}\|g_{t'}\|^2}\exp\left[\frac{1}{2}\left(\frac{L_t/L_{t-1}}{288\gamma k^2}\right)^{1/(2\gamma-1)}\right]
\end{align*}
so that
\begin{align*}
    \frac{\Lm}{T-1}\exp\left(\frac{(T-1)^{1/2\gamma}}{2(4k)^{1/\gamma}}\right)&=\max_{t\le T}\Lm\frac{L_{t-1}^2}{\sum_{t'=1}^{t-1}\|g_{t'}\|^2}\exp\left[\frac{1}{2}\left(\frac{L_t/L_{t-1}}{288\gamma k^2}\right)^{1/(2\gamma-1)}\right]
\end{align*}
Therefore we can put everything together to get
\begin{align*}
    \E[R_T(u)]&\ge k\|u\|\Lm \log^\gamma(T\|u\|+1)\sqrt{T} + \frac{\Lm}{T-1}\exp\left(\frac{(T-1)^{1/2\gamma}}{2(4k)^{1/\gamma}}\right)\\
    &\ge k\|u\|\Lm \log^\gamma(T\|u\|+1)\sqrt{T} + \max_{t\le T}\Lm\frac{L_{t-1}^2}{\sum_{t'=1}^{t-1}\|g_{t'}\|^2}\exp\left[\frac{1}{2}\left(\frac{L_t/L_{t-1}}{288\gamma k^2}\right)^{1/(2\gamma-1)}\right]
\end{align*}

\end{proof}
\section{FTRL regret}
We prove a general bound on the regret of FTRL. Our bound is not fundamentally tighter than the many previous analyses of FTRL, but we decompose the regret in a new way that makes our analysis much easier. We make use of ``shadow regularizers", $\psi^+_t$ that can be used to characterize regret more easily. Our bound bears some similarity in form to the adaptive online mirror descent bound of \citep{orabona2014generalized} and the analysis of FTRL with varying regularizers of \citep{cutkosky2016online}.

\begin{Theorem}\label{thm:ftrlmagic}
Let $\ell_t,\dots,\ell_T$ be an arbitrary sequence of loss functions. Define $\ell_0(w)=0$ for notational convenience. Let $\psi_0,\psi_1,\dots,\psi_{T-1}$ be a sequence of regularizer functions, such that $\psi_t$ is chosen without knowledge of $\ell_{t+1},\dots,\ell_T$. Let $\psi^+_1,\dots,\psi^+_T$ be an arbitrary sequences of regularizer functions (possibly chosen \emph{with} knowledge of the full loss sequence). Define $w_1,\dots,w_T$ to be the outputs of FTRL with regularizers $\psi_t$: $w_{t+1}=\argmin \psi_t +\ell_{1:t}$, and define $w^+_t$ for $t=2,\dots,T+1$ by $w^+_{t+1} = \argmin \psi^+_t + \ell_{1:t}$ Then FTRL with regularizers $\psi_t$ obtains regret
\begin{align*}
    R_T(u)&=\sum_{t=1}^T \ell_t(w_t)-\ell_t(u)\\
    &\le \psi^+_T(u)-\psi_0(w^+_2)+\sum_{t=1}^T \psi_{t-1}(w^+_{t+1})-\psi^+_t(w^+_{t+1})+\ell_t(w_t)-\ell_t(w^+_{t+1})\\
    &\quad\quad + \sum_{t=1}^{T-1} \psi^+_t(w^+_{t+2})-\psi_t(w^+_{t+2})
\end{align*}
\end{Theorem}

\begin{proof}

We define $X_t=w^+_{t+2}$ for $t<T$ and $X_T=u$. We'll use the symbols $X_t$ as intermediate variables in our proof in an attempt to keep the algebra cleaner. By definition of $w^+_{t+1}$, for all $t\le T$ we have
\begin{align*}
    \psi^+_t(w^+_{t+1})+\ell_{1:t}(w^+_{t+1})&\le \psi^+_t(X_t)+\ell_{1:t}(X_t)\\
    \ell_t(w_t)&\le \ell_t(w_t)+\ell_{1:t}(X_t)-\ell_{1:t}(w^+_{t+1}) + \psi^+_t(X_t)-\psi^+_t(w^+_{t+1})\\
    &= \ell_t(w_t)-\ell_t(w^+_{t+1})+\ell_{1:t}(X_t)-\ell_{1:t-1}(w^+_{t+1})\\
    &\quad\quad+\psi_{t-1}(w^+_{t+1})-\psi^+_t(w^+_{t+1})\\
    &\quad\quad+\psi^+_t(X_t)-\psi_{t-1}(w^+_{t+1})
\end{align*}
Summing this inequality across all $t$ we have
\begin{align*}
    \sum_{t=1}^T \ell_t(w_t)&\le \sum_{t=1}^T \ell_t(w_t)-\ell_t(w^+_{t+1}) \\
    &\quad\quad + \sum_{t=1}^T\ell_{1:t}(X_t)-\sum_{t=1}^T\ell_{1:t-1}(w^+_{t+1})\\
    &\quad\quad+\sum_{t=1}^T\psi_{t-1}(w^+_{t+1})-\psi^+_t(w^+_{t+1})\\
    &\quad\quad + \sum_{t=1}^T \psi^+_t(X_t)-\psi_{t-1}(w^+_{t+1})
\end{align*}

Notice that $\sum_{t=1}^T\ell_{1:t-1}(w^+_{t+1})=\sum_{t=2}^T\ell_{1:t-1}(w^+_{t+1})$ since the first term is zero. Thus after some re-indexing we have
\begin{align*}
    \sum_{t=1}^T \ell_t(w_t)&\le \sum_{t=1}^T \ell_t(w_t)-\ell_t(w^+_{t+1}) \\
    &\quad\quad + \ell_{1:T}(X_T)+\sum_{t=2}^T\ell_{1:t-1}(X_{t-1})-\sum_{t=2}^T\ell_{1:t-1}(w^+_{t+1})\\
    &\quad\quad+\sum_{t=1}^T\psi_{t-1}(w^+_{t+1})-\psi^+_t(w^+_{t+1})\\
    &\quad\quad + \psi^+_T(X_T)-\psi_0(w^+_2)+\sum_{t=1}^{T-1} \psi^+_t(X_t)-\sum_{t=1}^{T-1}\psi_{t}(w^+_{t+2})
\end{align*}
Now we substitute our values of $X_t=w^+_{t+2}$ for $t<T$ and $X_T=u$ to obtain
\begin{align*}
    \sum_{t=1}^T \ell_t(w_t)&\le \sum_{t=1}^T \ell_t(w_t)-\ell_t(w^+_{t+1}) \\
    &\quad\quad + \ell_{1:T}(u) + \psi^+_T(u)-\psi_0(w^+_2)\\
    &\quad\quad+\sum_{t=1}^T\psi_{t-1}(w^+_{t+1})-\psi^+_t(w^+_{t+1})\\
    &\quad\quad+\sum_{t=1}^{T-1} \psi^+_t(w^+_{t+2})-\sum_{t=1}^{T-1}\psi_{t}(w^+_{t+2})
\end{align*}

so that subtracting $\ell_{1:T}(u)$ from both sides we get a regret bound:
\begin{align*}
    R_T(u)&=\sum_{t=1}^T \ell_t(w_t)-\ell_t(u)\\
    &\le \sum_{t=1}^T \ell_t(w_t)-\ell_t(w^+_{t+1}) \\
    &\quad\quad + \psi^+_T(u)-\psi_0(w^+_2)\\
    &\quad\quad+\sum_{t=1}^T\psi_{t-1}(w^+_{t+1})-\psi^+_t(w^+_{t+1})\\
    &\quad\quad+\sum_{t=1}^{T-1} \psi^+_t(w^+_{t+2})-\sum_{t=1}^{T-1}\psi_{t}(w^+_{t+2})
\end{align*}
\end{proof}

\section{Facts About Strong Convexity}

In this section we prove some basic facts about our generalized strong convexity.

\begin{Proposition}\label{thm:strongconvexfacts}
Suppose $\psi:W\to \R$ is $\sigma$-strongly convex. Then:
\begin{enumerate}
\item $\psi+f$ is $\sigma$-strongly convex for any convex function $f$.
\item $c\psi$ is $c\sigma$-strongly convex for any $c\ge 0$.
\item Suppose $c\ge 0$ and $\phi(w)=\psi(cw)$. Let $\sigma'(x,y)=\sigma(cx,cy)$. Then $\phi$ is $c^2\sigma'$-strongly convex.
\end{enumerate}
\end{Proposition}

\begin{proof}
\begin{enumerate}
\item Let $x,y\in W$ and let $g\in \partial \psi(x)$ and $b\in \partial f(x)$. Then $g+b\in\partial (\psi+f)(x)$. By convexity and strongly convexity respectively we have:
\begin{align*}
\psi(y)&\ge \psi(x) + g\cdot(y-x) + \frac{\sigma(x,y)}{2}\|x-y\|^2\\
f(y)&\ge f(x)+b\cdot(y-x)
\end{align*}
so that adding these equations shows that $\psi+f$ is $\sigma$-strongly convex.

\item This follows immediately by multiplying the defining equation for strong convexity of $\psi$ by $c$.

\item Let $x,y\in W$ and let $g\in \partial \psi(cx)$. Then $cg\in\partial \phi(x)$.
\begin{align*}
    \psi(cy)&\ge \psi(cx) + g\cdot(cy-cx) + \frac{\sigma(cx,cy)}{2}\|cx-cy\|^2\\
    \phi(y)&\ge \phi(x)+cg\cdot(y-x)+\frac{\sigma(cx,cy)}{2}c^2\|x-y\|^2
\end{align*}
\end{enumerate}
\end{proof}
Note that for any linear function $f(w)=g\cdot w$, if $\psi$ is $\sigma$-strongly convex, then $\psi + f$ is also $\sigma$-strongly convex.

We show that the following lemma from \citep{mcmahan2014survey} about strongly-convex functions continues to hold under our more general definition. The proof of this lemma (and the next) are identical to the standard ones, but we include them here for completeness.
\begin{Lemma}\label{thm:strongconvextostability}
Suppose $A$ and $B$ are arbitrary convex functions such that $A+B$ is $\sigma$-strongly convex. Let $w_1=\argmin A$ and $w_2=\argmin A+B$ and let $g\in\partial B(w_1)$. Then
\begin{align*}
\|w_1-w_2\|\le \frac{\|g\|_\star}{\sigma(w_1,w_2)}
\end{align*}
\end{Lemma}
\begin{proof}
Since $w_2\in \argmin A+B$, we have $0\in \partial (A+B)(w_2)$ and so by definition of strong convexity we have
\begin{align*}
\frac{\sigma(w_1,w_2)}{2}\|w_1-w_2\|^2 \le A(w_1)+B(w_1)-A(w_2)-B(w_2)
\end{align*}

Now let $g\in \partial B(w_1)$. Consider the function $\hat A(w) = A(w)+B(w)-\langle g,w\rangle$. Then we must have $0\in \partial \hat A(w_1)$ and so by strong-convexity again we have
\begin{align*}
\frac{\sigma(w_1,w_2)}{2}\|w_1-w_2\|^2 \le A(w_2)+B(w_2)-\langle g,w_2\rangle-A(w_1)-B(w_1)+\langle g,w_1\rangle
\end{align*}

Adding these two equations yields:
\begin{align*}
\sigma(w_1,w_2)\|w_1-w_2\|^2 \le \langle g,w_1-w_2\rangle\le \|g\|_\star \|w_1-w_2\|
\end{align*}
and so we obtain the desired statement.
\end{proof}

Finally, we have an analog of a standard way to check for strong-convexity:
\begin{Proposition}\label{thm:hessiancondition}
Suppose $\psi:W\to \R$ is twice-differentiable and $v^T\nabla^2\psi(x)v \ge \sigma(x)\|v\|^2$ for all $x$ and $v$ for some norm $\|\cdot\|$ and $\sigma:W\to \R$ where $\sigma(x+ t(y-x))\ge \min(\sigma(x),\sigma(y))$ for all $x,y\in W$ and $t\in[0,1]$. Then $\psi$ is $\sigma$-strongly convex with respect to the norm $\|\cdot\|$.
\end{Proposition}
\begin{proof}
We integrate the derivative:
\begin{align*}
\psi(x)-\psi(y) &= \int_0^1 \frac{d}{dt}\psi(x+t(y-x))dt\\
&=\int_0^1 \nabla \psi(x+t(y-x))\cdot (y-x)dt\\
&= \nabla \psi(x)\cdot (y-x)\\
&\quad\quad\quad+\int_0^1\int_0^t(y-x)^T\nabla^2\psi(x+k(y-x))(y-x)dkdt\\
&\ge \nabla \psi(x)\cdot (y-x)+\|y-x\|^2\int_0^1\int_0^t \sigma(x+k(y-x))dkdt\\
&\ge \nabla \psi(x)\cdot (y-x)+\|y-x\|^2\int_0^1 t\min(\sigma(x),\sigma(y))dt\\
&=\nabla \psi(x)\cdot (y-x)+ \frac{\min(\sigma(x),\sigma(y))}{2}\|y-x\|^2
\end{align*}
\end{proof}
\section{Proof of Theorem \ref{thm:onedimconditions}}

First we prove a proposition that allows us to generate a strongly convex function easily:

\begin{Proposition}\label{thm:paramfreetostrong}
Suppose $\phi:\R\to \R$ is such that $\frac{\phi'(x)}{x}\ge \phi''(x)\ge 0$ and $\phi'''(x)\le 0$ for all $x\ge 0$. Let $W$ be a Hilbert Space and $\psi:W\to \R$ be given by $\psi(w)=\phi(\|w\|)$. Then $\psi$ is $\phi''(\|w\|)$-strongly convex with respect to $\|\cdot\|$.
\end{Proposition}
\begin{proof}
Let $x,y\in W$. We have
\begin{align*}
\nabla \psi(x) &= \phi'(\|x\|)\frac{x}{\|x\|}\\
\nabla^2 \psi(x) &= \left(\phi''(\|x\|)-\frac{\phi'(\|x\|)}{\|x\|}\right)\frac{xx^T}{\|x\|^2}+\frac{\phi'(\|x\|)}{\|x\|}I\\
&\succeq \phi''(\|x\|)I
\end{align*}
Where the last line follows since $\frac{\phi'(x)}{x}\ge \phi''(x)$ for all $x\ge 0$.
Since $\phi'''(x)\le 0$, $\phi''(x)$ is always decreasing for positive $x$ and so we have
\[
\phi''(\|x+t(y-x)\|)\ge \min(\phi''(\|x\|),\phi''(\|y\|))
\]
for all $t\in [0,1]$.
Therefore we can apply Proposition \ref{thm:hessiancondition} to conclude that $\psi$ is $\phi''(\|w\|)$-strongly convex.
\end{proof}

Now we prove Proposition \ref{thm:onedimconditions}, restated below:
\onedimconditions*
\begin{proof}

It's clear that $\psi(0)=0$ so the first condition for being an adaptive regularizer is satisfied.

Next we will show that $\frac{\phi'(x)}{x} \ge \phi''(x)$ so that we can apply Proposition \ref{thm:paramfreetostrong}. It suffices to show 
\begin{align*}
\phi'(x)- x\phi''(x)\ge 0
\end{align*}
Clearly this identity holds for $x=0$. Differentiating the right-hand-side of the equation, we have 
\begin{align*}
    \phi''(x)-x\phi'''(x)-\phi''(x) = -x\phi'''(x)\ge 0
\end{align*}
since $\phi'''(x)\le 0$ and $x\ge 0$. Thus $\phi'(x)-x\phi''(x)$ is non-decreasing and so must always be non-negative.

Therefore, by Proposition \ref{thm:paramfreetostrong}, $\psi$ is $(\phi''(\|\cdot\|),\|\cdot\|)$-strongly convex. Also, since $\phi'''(x)\le 0$, $\phi''(\|x\|)\le \phi''(\|y\|)$ when $\|x\|\ge \|y\|$ so that $\psi$ satisfies the second condition for being an adaptive regularizer.

Finally, observe that $\lim_{x\to\infty} \phi(x)\phi''(x)$ implies by definition that for any $C$ there exists a $B$ such that $\phi(x)\phi''(x)\ge C$ whenever $x\ge B$. Therefore we immediately see that $\psi(x)\phi''(\|x\|)\ge C$ for all $\|x\|\ge B$ so that the third condition is satified.
\end{proof}

\section{Proof of Theorem \ref{thm:parameterfreeregret}}

First we define new regularizers $\psi_t^+$ analagously to $\psi_t$ that we will use in conjunction with Theorem \ref{thm:ftrlmagic}:
\begin{Definition}\label{dfn:shadowetas}
Given a norm $\|\cdot\|$ and a sequence of subgradients $g_1,\dots,g_T$, define $L_t$ and $\frac{1}{\eta_t}$ as in Definition \ref{dfn:regularizers}, and define $L_0=L_1$. We define $\frac{1}{\eta_t^+}$ recursively by:
\begin{align*}
    \frac{1}{\eta_0^+}&=\frac{1}{\eta_0}\\
    \frac{1}{(\eta_t^+)^2}&=\max\left(\frac{1}{\eta_{t-1}^2}+2\|g_t\|_\star \min(\|g_t\|_\star, L_{t-1}),L_{t-1}\|g_{1:t}\|_\star\right)
\end{align*}

Further, given a $k\ge 1$ and a non-decreasing sequence of positive numbers $a_t$, define $\psi_t^+$ by:
\begin{align*}
    \psi^+_t(w)&=\frac{k}{\eta^+_t a_{t-1}}\psi(a_{t-1} w)\\
    w^+_{t+1} &= \argmin_{w\in W} \psi^+_t(w)+g_{1:t}\cdot w
\end{align*}
\end{Definition}

Throughout the following arguments we will assume $\eta_t$ and $\eta_t^+$ are the sequences defined in Definitions \ref{dfn:regularizers} and \ref{dfn:shadowetas}.

%Further, we will now redefine $\psi_t$ and $\psi_t^+$ to have value $\infty$ outside of $W$. This has no effect on the sequences $w_t$ produced by FTRL, but it makes some analysis easier in terms of subgradients at the boundary of $W$. In this spirit, when we write $\psi(a_{t-1} w)$, it is always a shorthand for $\psi(a_{t-1} w) + I_W(w)$ where $I_W(w)=0$ if $w\in W$ and $\infty$ otherwise. We feel that carrying the extra $I_W$s through our algebra is mechanical and impedes intuition.

%that all functions we consider ar when computing $w_{t+1}$, we replace $\psi$ with $\psi(w) + I_{a_tW}(w)$, where $I_{a_tW}(w)=0$ if $w/a_t\in W$ and $\infty$ otherwise. Observe that since $\psi_t(w) \propto \psi(a_t w)$, this has no effect on the sequence $w_t$ output by the FTRL algorithm.

The next proposition establishes several identities that we will need in proving our bounds.

\begin{Proposition}\label{thm:etarates}
Suppose $\psi$ is a $(\sigma,\|\cdot\|)$-adaptive regularizer, and $g_1,\cdots,g_T$ be some sequence of subgradients. Then the following identities hold:

\begin{enumerate}

\item
\begin{align*}
2\|g_t\|_\star L_{t-1}\eta^+_t\ge \left(\frac{1}{\eta^+_t} -\frac{1}{\eta_{t-1}}\right) \ge \|g_t\|_\star\min(\|g_t\|_\star,L_{t-1})\eta^+_t
\end{align*}

\item
%\begin{align*}
%    \frac{1}{\eta_t} & \le \sqrt{2(\|g\|_\star^2)_{1:t}+\Lm\max_{t'\le t}\|g_{1:t'}\|_\star}\\
%    &\le\sqrt{2\Lm(\|g\|_\star)_{1:t}}
%\end{align*}
\begin{align*}
    \frac{1}{\eta_t} &\le\sqrt{2L_t(\|g\|_\star)_{1:t}}\\
    \frac{1}{\eta_t} &\le\sqrt{2(\|g\|_\star^2)_{1:t}+\Lm\max_{t'\le t}\|g_{1:t'}\|_\star}
\end{align*}

\item
\begin{align*}
\|w_t-w^+_{t+1}\|&\le \frac{\|g_t\|_\star\eta^+_{t}+\left(\frac{1}{\eta^+_t}-\frac{1}{\eta_{t-1}}\right)\frac{1}{L_{t-1}}}{ a_{t-1}k\sigma(a_{t-1}w_t,a_{t-1} w^+_{t+1})}\\
\end{align*}

\item Let $\hat \psi$ be such that $\hat \psi(a_{t-1} w) =\psi(a_{t-1} w)$ for $w\in W$ and $\hat \psi(a_{t-1} w) =\infty$ for $w\notin W$. There exists some subgradient of $\hat \psi$ at $a_{t-1}w_t$, which with mild abuse of notation we call $\nabla \psi(a_{t-1}w_t)$, such that:
\begin{align*}
    |\nabla \hat \psi (a_{t-1}w_t)\cdot (w_t-w^+_{t+1})|&\le 3\frac{\frac{\|g_t\|_\star}{L_{t-1}}}{a_{t-1}k^2\sigma(a_{t-1}w_t,a_{t-1}w^+_{t+1})}
\end{align*}
\item
\begin{align*}
g_t\cdot(w_t-w^+_{t+1})&\le \frac{\|g_t\|^2_\star\eta^+_{t}+\left(\frac{1}{\eta_{t-1}}-\frac{1}{\eta^+_t}\right)\frac{\|g_t\|_\star}{L_{t-1}}}{a_{t-1}k\sigma(a_{t-1}w_t,a_{t-1}w^+_{t+1})}
\end{align*}

\item
\begin{align*}
    \frac{1}{\eta^+_t} &\le \sqrt{2\Lm(\|g\|_\star)_{1:T-1}+2\Lm L_{t-1}}
\end{align*}
\end{enumerate}

\end{Proposition}

\begin{proof}
Let $\hat \psi$ be such that $\hat \psi(a_{t-1} w) =\psi(a_{t-1} w)$ for $w\in W$ and $\hat \psi(a_{t-1} w) =\infty$ for $w\notin W$. Then we can write $w_t=\argmin_{w\in W}\frac{k}{a_{t-1}\eta_{t-1}}\psi(a_{t-1}w)+g_{1:t-1}\cdot w = \argmin \frac{k}{a_{t-1}\eta_{t-1}}\hat \psi(a_{t-1}w)+g_{1:t-1}$. From this it follws that there is some subgradient of $\hat \psi$ at $a_{t-1}w_t$, which we refer to (by mild abuse of notation) as $\nabla \hat \psi(a_{t-1}w_t)$ such that
\begin{align*}
\nabla \hat \psi(a_{t-1}w_t) &= -\frac{\eta_{t-1}g_{1:t-1}}{k}\\
\end{align*}
Note that we must appeal to a subgradient rather than the actual gradient in order to encompass the case that $a_{t-1}w_t$ is on the boundary of $W$.

Next, observe that
\begin{align*}
    \eta^+_t\eta_{t-1}\|g_{1:t-1}\|_\star &\le (\eta_{t-1})^2\|g_{1:t-1}\|_\star\le \frac{1}{L_{t-1}}
\end{align*}

Now we are ready to prove the various parts of the Proposition.

\begin{enumerate}
\item
By definition of $\eta_{t-1}$ and $\eta^+_t$ we have
\begin{align*}
    \frac{1}{(\eta^+_t)^2}-\frac{1}{(\eta_{t-1})^2}&\ge 2\|g_t\|_\star\min(\|g_t\|_\star,L_{t-1})\\
    \left(\frac{1}{\eta^+_t}-\frac{1}{\eta_{t-1}}\right)\left(\frac{1}{\eta^+_t}+\frac{1}{\eta_{t-1}}\right) &\ge 2\|g_t\|_\star\min(\|g_t\|_\star,L_{t-1})\\
    \left(\frac{1}{\eta^+_t} -\frac{1}{\eta_{t-1}}\right)\left(1+\frac{\eta^+_t}{\eta_{t-1}}\right) &\ge 2\|g_t\|_\star\min(\|g_t\|_\star,L_{t-1})\eta^+_t\\
    \frac{1}{\eta^+_t} -\frac{1}{\eta_{t-1}}& \ge \|g_t\|_\star\min(\|g_t\|_\star,L_{t-1}) \eta^+_t
\end{align*}
where in the last line we used the fact that $\eta^+_t\le \eta_{t-1}$ to conclude that $1+\frac{\eta^+_t}{\eta_{t-1}}\le 2$.

For the other direction, we have two cases:
\begin{enumerate}
\item[1.] $\frac{1}{(\eta^+_t)^2} =\frac{1}{(\eta_{t-1})^2} + 2\|g_t\|_\star\min(\|g_t\|_\star,L_{t-1})$.
\item[2.] $\frac{1}{(\eta^+_t)^2} = L_{t-1}\|g_{1:t}\|_\star$.
\end{enumerate}

\noindent\textbf{Case 1 $\frac{1}{(\eta^+_t)^2} =\frac{1}{(\eta_{t-1})^2} + 2\|g_t\|_\star\min(\|g_t\|_\star,L_{t-1})$:}

In this case we have

\begin{align*}
    \frac{1}{(\eta^+_t)^2}-\frac{1}{(\eta_{t-1})^2}&= 2\|g_t\|_\star\min(\|g_t\|_\star,L_{t-1})\\
    \left(\frac{1}{\eta^+_t}-\frac{1}{\eta_{t-1}}\right)\left(\frac{1}{\eta^+_t}+\frac{1}{\eta_{t-1}}\right) &= 2\|g_t\|_\star\min(\|g_t\|_\star,L_{t-1})\\
    \left(\frac{1}{\eta^+_t} -\frac{1}{\eta_{t-1}}\right)\left(1+\frac{\eta^+_t}{\eta_{t-1}}\right) &= 2\|g_t\|_\star\min(\|g_t\|_\star,L_{t-1})\eta^+_t\\
    \frac{1}{\eta^+_t} -\frac{1}{\eta_{t-1}}& \le 2\|g_t\|_\star\min(\|g_t\|_\star,L_{t-1}) \eta^+_t
\end{align*}
where in the last line we used the fact that $1+\frac{\eta^+_t}{\eta_{t-1}}\ge 1$.

\noindent\textbf{Case 2 $\frac{1}{(\eta^+_t)^2} = L_{t-1}\|g_{1:t}\|_\star$:}

\begin{align*}
    \frac{1}{(\eta^+_t)^2}-\frac{1}{(\eta_{t-1})^2} &\le L_{t-1}\|g_{1:t}\|_\star-L_{t-1}\|g_{1:t-1}\|_\star\\
    &\le L_{t-1}\|g_t\|_\star\le L_{t-1}\|g_t\|_\star
\end{align*}

Now we follow the exact same argument as in Case 1 to show $\frac{1}{\eta^+_t} -\frac{1}{\eta_{t-1}} \le L_{t-1}\|g_t\|_\star \eta^+_t$, which proves the desired result.

\item
We proceed by induction for both claims. The statements are clear for $\frac{1}{\eta_1} = \sqrt{2}\|g_1\|_\star$. Suppose 
\begin{align*}
    \frac{1}{\eta_t} &\le \sqrt{2L_t(\|g\|_\star)_{1:t}}\\
    \frac{1}{\eta_t} &\le \sqrt{2(\|g\|_\star^2)_{1:t}+\Lm\max_{t'\le t}\|g_{1:t'}\|_\star}
\end{align*}
Then observe that $\frac{1}{\eta_t^2}+2\|g_{t+1}\|_\star^2\le 2L_{t+1}(\|g\|_\star)_{1:t+1}$ by the induction hypothesis, and $L_{t+1}\|g_{1:t+1}\|_\star \le 2L_{t+1}(\|g\|_\star)_{1:t+1}$. Therefore $\frac{1}{\eta_{t+1}}\le \sqrt{2L_{t+1}(\|g\|_\star)_{1:t+1}}$, proving the first claim.

The induction step for the second claim follows from the observations:
\begin{align*}
    2(\|g\|_\star^2)_{1:t+1}+\Lm\max_{t'\le t+1}\|g_{1:t'}\|_\star&\ge 2(\|g\|_\star^2)_{1:t}+\Lm\max_{t'\le t}\|g_{1:t'}\|_\star+2\|g_{t+1}\|_\star^2\\
    2(\|g\|_\star^2)_{1:t+1}+\Lm\max_{t'\le t+1}\|g_{1:t'}\|_\star&\ge L_t \|g_{1:t+1}\|_\star
\end{align*}
so that $\frac{1}{\eta_{t+1}}\le \sqrt{2(\|g\|_\star^2)_{1:t+1}+\Lm\max_{t'\le t+1}\|g_{1:t'}\|_\star}$ as desired.
\item
Let $I_{a_{t-1}W}(w)$ be the indicator of the set $a_{t-1}W$ - $I_{a_{t-1}W}(a_{t-1}w)=0$ if $w\in W$ and $\infty$ otherwise.
Observe that $\hat \psi(w) = \psi(w) + I_{a_{t-1}W}(w)$. Observe that $\hat \psi(w) = I_{a_{t-1}W}(w)+\psi(w)$.

Now the third equation follows from Lemma \ref{thm:strongconvextostability}, setting $A(w)=I_{a_{t-1}W}(w) + \frac{k}{a_{t-1}\eta_{t-1}}\psi(w)+\frac{g_{1:t-1}}{a_{t-1}}\cdot w$ and $B(w) =I_{a_{t-1}W}(w) + \frac{g_t}{a_{t-1}}\cdot w+\left(\frac{1}{a_{t-1}\eta^+_t}-\frac{k}{a_{t-1}\eta_{t-1}}\right)\psi(w)$. Then by inspection of the definitions of $w_t$ and $w^+_{t+1}$, we have $a_{t-1}w_t = \argmin A$ and $a_{t-1}w^+_{t+1} = \argmin A+B$. Further, by Corollary \ref{thm:strongconvexfacts}, $A+B$ is $\frac{k\sigma}{a_{t-1}\eta^+_t}$-strongly convex. We can re-write $A$ and $B$ in terms of $\hat \psi$ by simply replacing the $\psi$s with $\hat \psi$s and removing the $I_{a_{t-1}W}$s. Now we use the facts noted at the beginning of the proof:
\begin{align*}
    \nabla \hat \psi(a_{t-1}w_t) &= -\frac{\eta_{t-1}g_{1:t-1}}{k}\\
    \eta^+_t\eta_{t-1}\|g_{1:t-1}\|&\le\frac{1}{L_{t-1}}
\end{align*}
Applying these identities with Lemma \ref{thm:strongconvextostability} we have:
\begin{align*}
    \|a_{t-1}w_t-a_{t-1}w^+_{t+1}\|&\le a_{t-1}\eta^+_{t}\frac{\|\frac{g_t}{a_{t-1}}+\left(\frac{k}{a_{t-1}\eta^+_t}-\frac{k}{a_{t-1}\eta_{t-1}}\right)\nabla \hat\psi(a_{t-1}w_t)\|_\star}{k\sigma(a_{t-1}w_t,a_{t-1}w^+_{t+1})}\\
    &\le \frac{\eta^+_{t}\|g_t\|_\star}{\sigma(a_{t-1}w_t,a_{t-1}w^+_{t+1})}+\frac{\eta^+_t\left(\frac{k}{\eta^+_t}-\frac{k}{\eta_{t-1}}\right) \frac{\eta_{t-1}\|g_{1:t-1}\|_\star}{k}}{k\sigma(a_{t-1}w_t,a_{t-1}w^+_{t+1})}\\
    &\le \frac{\eta^+_{t}\|g_t\|_\star}{k\sigma(a_{t-1}w_t,a_{t-1}w^+_{t+1})}+\frac{\left(\frac{1}{\eta^+_t}-\frac{1}{\eta_{t-1}}\right) \frac{1}{L_{t-1}}}{k\sigma(a_{t-1}w_t,a_{t-1}w^+_{t+1})}\\
\end{align*}
And we divide by $a_{t-1}$ to conclude the desired identity.

\item 

Using the already-proved parts 1 and 3 of this Proposition and definition of dual norm, we have
\begin{align*}
    |\nabla \hat \psi (a_{t-1} w_t)\cdot(w_t-w^+_{t+1})|&\le \|\nabla \psi( a_{t-1} w_t)\|_\star\|w_t-w^+_{t+1}\|\\
    &\le  \frac{\eta_{t-1}\|g_{1:t-1}\|_\star}{k}\frac{\eta^+_{t}\|g_t\|_\star}{a_{t-1}k\sigma(a_{t-1}w_t,a_{t-1}w^+_{t+1})}\\
    &\quad\quad+\frac{\eta_{t-1}\|g_{1:t-1}\|_\star}{k}\frac{\left(\frac{1}{\eta^+_t}-\frac{1}{\eta_{t-1}}\right) \frac{1}{L_{t-1}}}{a_{t-1}k\sigma(a_{t-1}w_t,a_{t-1}w^+_{t+1})}\\
    &\le \frac{\frac{\|g_t\|_\star}{L_{t-1}}}{a_{t-1}k^2\sigma(a_{t-1}w_t,a_{t-1}w^+_{t+1})}+\frac{\eta^+_t \eta_{t-1}\|g_{1:t-1}\|_\star 2L_{t-1}\|g_t\|\frac{1}{L_{t-1}}}{a_{t-1}k^2\sigma(a_{t-1}w_t,a_{t-1}w^+_{t+1})}\\
    &\le  \frac{\frac{\|g_t\|_\star}{L_{t-1}}}{a_{t-1}k^2\sigma(a_{t-1}w_t,a_{t-1}w^+_{t+1})}+\frac{\frac{1}{L_{t-1}^2} 2L_{t-1}\|g_t\|_\star}{a_{t-1}k^2\sigma(a_{t-1}w_t,a_{t-1}w^+_{t+1})}\\
    &\le 3\frac{\frac{\|g_t\|_\star}{L_{t-1}}}{a_{t-1}k^2\sigma(a_{t-1}w_t,a_{t-1}w^+_{t+1})}
\end{align*}
\item
The fifth part of the Proposition follows directly from part 3 by the definition of dual norm.

\item
By part 2, we have
\begin{align*}
    \frac{1}{\eta_{t-1}} & \le \sqrt{2\Lm(\|g\|_\star)_{1:t-1}}
\end{align*}
We consider the two cases:

\noindent\textbf{Case 1 $\frac{1}{(\eta^+_t)^2} = \frac{1}{(\eta_{t-1})^2}+2\|g_t\|_\star\min(\|g_t\|_\star,L_{t-1})$:}
In this case we have
\begin{align*}
    \frac{1}{(\eta^+_t)^2} &\le 2\Lm(\|g\|_\star)_{1:t-1} + 2\|g_t\|_\star\min(\|g_t\|_\star,L_{t-1})\\
    &\le 2\Lm(\|g\|_\star)_{1:t-1} + + 2\Lm L_{t-1}
\end{align*}

\noindent\textbf{Case 2 $\frac{1}{(\eta^+_t)^2}=L_{t-1}\|g_{1:t}\|_\star$:}
\begin{align*}
    \frac{1}{(\eta^+_t)^2} &\le L_{t-1}\|g_{1:t}\|_\star\\
    &\le L_{t-1}\|g_{1:t-1}\|+L_{t-1}\|g_t\|\\
    &\le \Lm(\|g\|_\star)_{1:t-1} + \Lm L_{t-1}
\end{align*}

\end{enumerate}
\end{proof}

\begin{Lemma}\label{thm:highwtoratiobound}
Suppose $\psi$ a $(\sigma,\|\cdot\|)$-adaptive regularizer and $g_1,\cdots,g_T$ is some sequence of subgradients. We use the terminology of Definition \ref{dfn:regularizers}. Recall that we define $h(w)=\psi(w)\sigma(w)$ and $h^{-1}(x) = \max_{h(w)\le x} \|w\|$.
Suppose either of the follow holds:
\begin{enumerate}
\item $\|w^+_{t+1}\| \ge \frac{h^{-1}\left(2\frac{L_t}{k^2L_{t-1}}\right)}{a_{t-1}}$ and $\|w^+_{t+1}\|\ge \|w_t\|$.
%$\sigma(a_{t-1}(w^+_{t+1}-\overline{w_{t-1}}))\le \sigma(a_{t-1}(w_t-\overline{w_{t-1}}))$.
\item $\|w_t\|\ge \frac{h^{-1}\left(5\frac{L_t}{k^2L_{t-1}}\right)}{a_{t-1}}$ and $\|w_t\|\ge \|w^+_{t+1}\|$.
%$\sigma(a_{t-1}(w_t-\overline{w_{t-1}}))\le \sigma(a_{t-1}(w^+_{t+1}-\overline{w_{t-1}}))$.
\end{enumerate}
Then 
\begin{align*}
\psi_{t-1}(w^+_{t+1})-\psi^+_{t}(w^+_{t+1})+g_t(w_t-w^+_{t+1})\le 0
\end{align*}
\end{Lemma}
\begin{proof}

As in Proposition \ref{thm:etarates}, we use $\nabla \psi(x)$ to simply mean some particular subgradient of $\psi$ at $x$.

\noindent\textbf{Case 1: $\|w^+_{t+1}\| \ge \frac{h^{-1}\left(2\frac{L_t}{k^2L_{t-1}}\right)}{a_{t-1}}$ and $\|w^+_{t+1}\|\ge \|w_t\|$:}
%$\sigma(a_{t-1}w^+_{t+1})\le \sigma(a_{t-1}w_t)$:}

By definition of adaptive regularizer (part 2), we must have $\sigma(a_{t-1}w^+_{t+1})\le \sigma(a_{t-1}w_t)$ since $\|w^+_{t+1}\|\ge \|w_t\|$. Therefore $\sigma(a_{t-1}w^+_{t+1},a_{t-1}w_t)=\sigma(a_{t-1}w^+_{t+1})$.

By definition of $h$, when $\|w^+_{t+1}\|\ge \frac{h^{-1}\left(2 \frac{L_t}{k^2L_{t-1}}\right)}{a_{t-1}}$ we can apply Proposition \ref{thm:etarates} (parts 1 and 5) to obtain
\begin{align*}
    \psi(a_{t-1}w^+_{t+1})\sigma(a_{t-1}w^+_{t+1})& \ge 2\frac{L_t}{k^2L_{t-1}}\\
    \left(\frac{1}{a_{t-1}\eta^+_t} -\frac{1}{a_{t-1}\eta_{t-1}}\right)\psi(a_{t-1}w^+_{t+1})&\ge\frac{ \left(\frac{1}{a_{t-1}\eta^+_t} -\frac{1}{a_{t-1}\eta_{t-1}}\right)2\frac{L_t}{L_{t-1}}}{k^2\sigma(a_{t-1}w_t,a_{t-1}w^+_{t+1})}\\
    \left(\frac{k}{a_{t-1}\eta^+_t} -\frac{k}{a_{t-1}\eta_{t-1}}\right)\psi(a_{t-1}w^+_{t+1})&\ge \frac{\left(\frac{1}{\eta^+_t} -\frac{1}{\eta_{t-1}}\right)}{a_{t-1}k\sigma(a_{t-1}w_t,a_{t-1}w^+_{t+1})}2\frac{L_t}{L_{t-1}}\\
    \psi^+_t(w^+_{t+1})-\psi_{t-1}(w^+_{t+1}) &\ge \frac{\|g_t\|_\star\min(\|g_t\|_\star,L_{t-1})\eta^+_t\frac{L_t}{L_{t-1}}+\left(\frac{1}{\eta^+_t}
     -\frac{1}{\eta_{t-1}}\right)\frac{L_t}{L_{t-1}}}{a_{t-1}k\sigma(a_{t-1}w_t,a_{t-1}w^+_{t+1})}\\
     &\ge \frac{\|g_t\|^2_\star\eta^+_t+\left(\frac{1}{\eta^+_t}
     -\frac{1}{\eta_{t-1}}\right)\frac{\|g_t\|_\star}{L_{t-1}}}{a_{t-1}k\sigma(a_{t-1}w_t,a_{t-1}w^+_{t+1})}\\
    &\ge g_t\cdot(w_t-w^+_{t+1})
\end{align*}

We remark that in the calculations above, we showed 
\begin{align*}
    \frac{\left(\frac{1}{\eta^+_t} -\frac{1}{\eta_{t-1}}\right)2\frac{L_t}{L_{t-1}}}{a_{t-1}\sigma(a_{t-1}kw_t,a_{t-1}w^+_{t+1})}\ge g_t(w_t-w_{t+1}^+)
\end{align*}
which we will re-use in Case 2.

\noindent\textbf{Case 2 $\|w_t\|\ge \frac{h^{-1}\left(5\frac{\|g_t\|_\star}{k^2L_{t-1}}\right)}{a_{t-1}}$, and $\|w_t\|\ge \|w^+_{t+1}\|$:}
%$\sigma(a_{t-1}w_t)\ge \sigma(a_{t-1}w^+_{t+1})$:}

Again, by definition of adaptive regularizer (part 2), we must have $\sigma(a_{t-1}w^+_{t+1})\ge \sigma(a_{t-1}w_t)$ since $\|w^+_{t+1}\|\le \|w_t\|$. Therefore $\sigma(a_{t-1}w^+_{t+1},a_{t-1}w_t)=\sigma(a_{t-1}w_t)$.
%In this case the definition of $(k,M,\sigma,\|\cdot\|)$-adaptive regularizer gives $\|\psi'(w_t)\|_\star\le kh(w_t)$. Therefore we have:
%\begin{align*}
%\left(\frac{1}{\eta_t}-\frac{1}{\eta_{t-1}}\right)\|\psi'(w_{t})\|_\star&\le \|g_t\|_\star^2\eta_t \|\phi'(w_{t})\|_\star\\
%&\le k\eta_t\|g_t\|_\star^2h(w_{t})
%\end{align*}
%Further, since by definition of the FTRL update we have
%\begin{align*}
%\|\phi'(w_t)\|_\star= \eta_{t-1}\|g_{1:t-1}\|_\star
%\end{align*}
%so that we have
%\begin{align*}
%2k\frac{\|g_t\|_\star}{L_{t-1}} &\le  \eta_{t-1}\|g_{1:t-1}\|_\star\\
%\eta_{t-1}\|g_t\|_\star &\le \eta_{t-1}^2 \frac{\|g_{1:t-1}\|_\star}{L_{t-1}}\frac{1}{2k}\le \frac{1}{2k}
%\end{align*}
Let $\hat \psi$ be as in Proposition \ref{thm:etarates} part 4. Oberve that $w_{t+1}^+$ and $w_t$ are both in $W$, so that we have $\psi(a_{t-1} w^+_{t+1}) = \hat \psi(a_{t-1} w^+_{t+1})$ and $\psi(a_{t-1} w_t) = \hat \psi(a_{t-1} w_t)$. Then we have:
\begin{align*}
\psi^+_t(w^+_{t+1})-\psi_{t-1}(w^+_{t+1})&= \left(\frac{k}{a_{t-1}\eta^+_t} -\frac{k}{a_{t-1}\eta_{t-1}}\right) \psi(a_{t-1}w^+_{t+1})\\
&= \left(\frac{k}{a_{t-1}\eta^+_t} -\frac{k}{a_{t-1}\eta_{t-1}}\right) \hat \psi(a_{t-1}w^+_{t+1})\\
&\ge \left(\frac{k}{a_{t-1}\eta^+_t} -\frac{k}{a_{t-1}\eta_{t-1}}\right) \left(\hat \psi(a_{t-1}w_{t}) - \left|a_{t-1}\nabla \hat \psi(a_{t-1}w_t)\cdot(w^+_{t+1}-w_t)\right|\right)\\
&\ge \left(\frac{k}{\eta^+_t} -\frac{k}{\eta_{t-1}}\right) \left(\frac{\psi(a_{t-1}w_t)}{a_{t-1}} -3\frac{\frac{\|g_t\|_\star}{L_{t-1}}}{a_{t-1}k^2\sigma(a_{t-1}w_t,a_{t-1}w^+_{t+1})}\right)\\
&\ge \left(\frac{k}{\eta^+_t} -\frac{k}{\eta_{t-1}}\right) \left(\frac{\psi(a_{t-1}w_t)}{a_{t-1}} -3\frac{\frac{L_t}{L_{t-1}}}{a_{t-1}k^2\sigma(a_{t-1}w_t,a_{t-1}w^+_{t+1})}\right)]
%&\ge  \eta_t\|g_t\|_\star^2 \psi(w_t) - k\eta_t\|g_t\|_\star^2h(w_t)\frac{\eta_{t-1}\|g_t\|_\star}{\sigma(w_t,w_{t+1})}\\
%&\ge \eta_t\|g_t\|_\star^2 \psi(w_t) -k\eta_t\|g_t\|_\star^2\psi(w_t)\eta_{t-1}\|g_t\|_\star\\
%&\ge \eta_t\|g_t\|_\star^2 \psi(w_t) -\frac{\eta_t\|g_t\|_\star^2\psi(w_t)}{2}\\
%&=\frac{\eta_t\|g_t\|_\star^2\psi(w_t)}{2}
\end{align*}

Now by definition of $h$, when $\|w_t\|\ge \frac{h^{-1}(5\frac{L_t}{k^2L_{t-1}})}{a_{t-1}}$ we have
\begin{align*}
    \psi(a_{t-1}w_{t})\sigma(a_{t-1}w_{t})& \ge 5\frac{L_t}{k^2L_{t-1}}\\
    \left(\frac{k}{a_{t-1}\eta^+_t} -\frac{k}{a_{t-1}\eta_{t-1}}\right)\psi(a_{t-1}w_{t})&\ge \frac{\left(\frac{k}{\eta^+_t} -\frac{k}{\eta_{t-1}}\right)5\frac{L_t}{L_{t-1}}}{a_{t-1}k^2\sigma(a_{t-1}w_t,a_{t-1}w^+_{t+1})}\\
    \left(\frac{k}{\eta^+_t} -\frac{k}{\eta_{t-1}}\right)\left(\frac{\psi(a_{t-1}w_t)}{a_{t-1}}-3\frac{\frac{L_t}{L_{t-1}}}{a_{t-1}k^2\sigma(a_{t-1}w_t,a_{t-1}w^+_{t+1})}\right) &\ge \frac{\left(\frac{1}{\eta^+_t} -\frac{1}{\eta_{t-1}}\right)2\frac{L_t}{L_{t-1}}}{a_{t-1}k\sigma(a_{t-1}w_t,a_{t-1}w^+_{t+1})}\\
    \psi^+_t(w^+_{t+1})-\psi_{t-1}(w^+_{t+1})&\ge g_t\cdot(w_t-w^+_{t+1})
\end{align*}

\end{proof}

The next theorem is a general fact about adaptive regularizers that is useful for controlling $\psi^+_t-\psi_t$:
\begin{Proposition}\label{thm:rescaling}
Suppose $\psi:W\to \R$ is a $(\sigma,\|\cdot\|)$-adaptive regularizer. Then $\frac{\psi(aw)}{a}$ is an increasing function of $a$ for all $a>0$ for all $w\in W$.
\end{Proposition}
\begin{proof}
Let's differentiate: $\frac{d}{da} \frac{\psi(aw)}{a} = \frac{\nabla \psi(aw) \cdot w}{a} -\frac{\psi(aw)}{a^2}$. Thus it suffices to show
\begin{align*}
    \nabla \psi(aw)\cdot aw\ge \psi(aw)
\end{align*}
But this follows immediately from the definition of subgradient, since $\psi(0)=0$.
\end{proof}

\begin{restatable}{Lemma}{norecentering}\label{thm:norecentering}
Suppose $\psi$ is a $(\sigma,\|\cdot\|)$-adaptive regularizer and $g_1,\dots,g_T$ is an arbitrary sequence of subgradients (possibly chosen adaptively). Using the terminology of Definition \ref{dfn:regularizers},
\begin{align*}
    \psi^+_t(w^+_{t+2})-\psi_t(w^+_{t+1})\le 0
\end{align*}
for all $t$
\end{restatable}
\begin{proof}

This follows from the fact that $a_{t-1}\le a_t$, and property 4 of an adaptive regularizer ($\psi(ax)/a$ is a non-decreasing function of $a$). By Proposition \ref{thm:etarates} (part 1), we have $\frac{1}{\eta^+_t}\le \frac{1}{\eta_t}$. Therefore:
\begin{align*}
    \psi^+_t(w^+_{t+2})&=\frac{k}{\eta^+_t a_{t-1}}\psi(a_{t-1}w^+_{t+2})\\
    &\le \frac{k}{\eta_t a_{t-1}}\psi(a_{t-1}w^+_{t+2})\\
    &\le \frac{k}{\eta_t a_{t}}\psi(a_{t}w^+_{t+2})\\
    &=\psi_t(w^+_{t+2})
\end{align*}
\end{proof}

\begin{Lemma}\label{thm:sumbounds}
Suppose $\psi$ is a $(\sigma,\|\cdot\|)$-adaptive regularizer and $g_1,\dots,g_T$ is an arbitrary sequence of subgradients (possibly chosen adaptively). We use the regularizers of Definition \ref{dfn:regularizers}. Recall that we define $h(w)=\psi(w)\sigma(w)$ and $h^{-1}(x) = \argmax_{h(w)\le x}\|w\|$. Define 
\[
\sigmamin =  \inf_{\|w\|\le h^{-1}\left(10/k^2\right)} k\sigma(w)
\]
and
\[
D = 2\max_t \frac{h^{-1}\left(5\frac{L_t}{kL_{t-1}}\right)}{a_{t-1}}
\]
Then
\begin{align*}
    &\psi_{t-1}(w^+_{t+1})-\psi^+_{t}(w^+_{t+1})+g_t(w_t-w^+_{t+1})\\
    &\le
    \left\{\begin{array}{ll}
    \|g_t\|_\star \min(D, \max_t(\|w_t-w^+_{t+1}\|))&\text{ when }\|g_t\|> 2L_{t-1}\\
    \frac{3\|g_t\|_\star^2\eta^+_t}{a_{t-1}\sigmamin}&\text{ otherwise}
    \end{array}\right.
\end{align*}

\end{Lemma}

\begin{proof}

By Lemma \ref{thm:highwtoratiobound}, whenever either $\|w^+_{t+1}\|\ge \frac{h^{-1}\left(5\frac{L_t}{k^2L_{t-1}}\right)}{a_{t-1}}\ge \frac{h^{-1}\left(2\frac{L_t}{k^2L_{t-1}}\right)}{a_{t-1}}$ or $\|w_{t}\|\ge \frac{h^{-1}\left(5\frac{L_t}{k^2L_{t-1}}\right)}{a_{t-1}}$ we must have
\begin{align*}
\psi_{t-1}(w^+_{t+1})-\psi^+_{t}(w^+_{t+1})+g_t(w_t-w^+_{t+1})\le 0
\end{align*}
Therefore, we have:
\begin{align*}
\psi_{t-1}(w^+_{t+1})-\psi^+_{t}(w^+_{t+1})+g_t(w_t-w^+_{t+1})&\le \left\{\begin{array}{ll}
g_t\cdot(w_t-w^+_{t+1})&\text{ when }\max(\|w_t\|,\|w^+_{t+1}\|)\le \frac{h^{-1}\left(5\frac{L_t}{k^2L_{t-1}}\right)}{a_{t-1}}\\
0&\text{ otherwise}
\end{array}
\right.
\end{align*}
When $\|g_t\|_\star\le 2L_{t-1}$, then we have $h^{-1}\left(5\frac{L_t}{k^2L_{t-1}}\right)\le h^{-1}(10/k^2)$. Thus when $\max(\|w_t\|,\|w_{t+1}^+\|)\le \frac{h^{-1}\left(5\frac{L_t}{k^2L_{t-1}}\right)}{a_{t-1}}$ and $\|g_t\|_\star \le 2L_{t-1}$, by Proposition \ref{thm:etarates} (part 5), we have
\begin{align*}
g_t(w_t-w_{t+1}^+)&\le \frac{\|g_t\|^2_\star\eta^+_{t}+\left(\frac{1}{\eta^+_t}-\frac{1}{\eta_{t-1}}\right)\frac{\|g_t\|_\star}{L_{t-1}}}{a_{t-1}\sigmamin}
\end{align*}
Therefore when $\|g_t\|_\star \le 2L_{t-1}$ we have (using Proposition \ref{thm:etarates} part 1):
\begin{align*}
\psi_{t-1}(w^+_{t+1})-\psi^+_{t}(w^+_{t+1})+g_t(w_t-w^+_{t+1})&\le \frac{\|g_t\|^2_\star\eta^+_{t}+\left(\frac{1}{\eta^+_t}-\frac{1}{\eta_{t-1}}\right)\frac{\|g_t\|_\star}{L_{t-1}}}{a_{t-1}\sigmamin}\\
&\le \frac{\|g_t\|_\star^2 \eta^+_{t}+2\|g_t\|_\star L_{t-1}\eta^+_t\frac{\|g_t\|_\star}{L_{t-1}}}{a_{t-1}\sigmamin}\\
&\le \frac{3\|g_t\|_\star^2\eta^+_t}{a_{t-1}\sigmamin}
\end{align*}

so that we can improve our conditional bound to:
\begin{align*}
    &\psi_{t-1}(w^+_{t+1})-\psi^+_{t}(w^+_{t+1})+g_t(w_t-w^+_{t+1})\\
    &\le \left\{\begin{array}{ll}
    g_t\cdot(w_t-w^+_{t+1})&\text{ when } \max(\|w_t\|,\|w^+_{t+1}\|)\le \frac{h^{-1}\left(5\frac{L_t}{k^2L_{t-1}}\right)}{a_{t-1}}\text{ and }\|g_t\|_\star> 2L_{t-1}\\
    \frac{3\|g_t\|^2\eta^+_t}{a_{t-1}\sigmamin}&\text{ when }\max(\|w_t\|,\|w^+_{t+1}\|)\le \frac{h^{-1}\left(5\frac{L_t}{k^2L_{t-1}}\right)}{a_{t-1}}\text{ and }\|g_t\|_\star\le 2L_{t-1}\\
    0&\text{ otherwise}
    \end{array}\right.
\end{align*}

When both $\|w^+_{t+1}\|$ and $\|w_t\|$ are less than than $ \frac{h^{-1}\left(5\frac{L_t}{k^2L_{t-1}}\right)}{a_{t-1}}$ then we also have

\begin{align*}
    \|w_t-w^+_{t+1}\| \le \min\left(D,\max_t \|w_t-w^+_{t+1}\|\right)
\end{align*}

where we define
\begin{align*}
D = 2\max_t \frac{h^{-1}\left(5 \frac{L_t}{k^2L_{t-1}}\right)}{a_{t-1}}
\end{align*}
Therefore we have
\begin{align*}
    &\psi_{t-1}(w^+_{t+1})-\psi^+_{t}(w^+_{t+1})+g_t(w_t-w^+_{t+1})\\
    &\le \left\{\begin{array}{ll}
    g_t\cdot(w_t-w^+_{t+1})&\text{ when } \max(\|w_t\|,\|w^+_{t+1}\|)\le \frac{h^{-1}\left(5\frac{L_t}{L_{t-1}}\right)}{a_{t-1}}\text{ and }\|g_t\|_\star> 2L_{t-1}\\
    \frac{3\|g_t\|^2\eta^+_t}{a_{t-1}\sigmamin}&\text{ when }\max(\|w_t\|,\|w^+_{t+1}\|)\le \frac{h^{-1}\left(5\frac{L_t}{k^2L_{t-1}}\right)}{a_{t-1}}\text{ and }\|g_t\|_\star\le 2L_{t-1}\\
    0&\text{ otherwise}
    \end{array}\right.\\
    &\le
    \left\{\begin{array}{ll}
    \|g_t\|_\star \min(D,\max_t \|w_t-w^+_{t+1}\|),&\text{ when }\|g_t\|> 2L_{t-1}\\
    \frac{3\|g_t\|_\star^2\eta^+_t}{a_{t-1}\sigmamin}&\text{ otherwise}
    \end{array}\right.
\end{align*}
\end{proof}

Now we have three more technical lemmas:

\begin{Lemma}\label{thm:doublingsum}
Let $a_1,\dots,a_M$ be a sequence of non-negative numbers such that $a_{i+1}\ge 2a_i$. Then
\begin{align*}
    \sum_{i=1}^M a_i\le 2a_M
\end{align*}
\end{Lemma}
\begin{proof}
We proceed by induction on $M$. For the base case, we observe that $a_1\le 2a_1$. Suppose $\sum_{i=1}^{M-1}a_i \le 2a_{M-1}$. Then we have
\begin{align*}
    \sum_{i=1}^M a_i &=a_M+\sum_{i=1}^{M-1} a_i\\
    &\le a_M + 2a_{M-1}\\
    &\le a_M + a_M=2a_M
\end{align*}
\end{proof}

The next lemma establishes some identities analogous to the bounds $\sum_{t=1}^T \frac{1}{\sqrt{t}} = O(\sqrt{T})$, and $\sum_{t=1}^T\frac{1}{T^2} = O(1)$. These are useful for dealing with increasing $a_t$ in our regret bounds.
\begin{Lemma}\label{thm:getasum}
\begin{enumerate}
\item
\begin{align*}
\sum_{t|\ \|g_t\|_\star\le 2L_{t-1}} \|g_t\|^2_\star \eta^+_t&\le \frac{2}{\eta^+_T}
\end{align*}

\item
Suppose $\alpha_t$ is defined by
\begin{align*}
    \alpha_0 &= \frac{1}{(L_1\eta_1)^2}\\
    \alpha_t &= \max\left(\alpha_{t-1},\frac{1}{(L_t\eta_t)^2}\right)
\end{align*}
then 
\begin{align*}
\sum_{t|\ \|g_t\|_\star\le 2L_{t-1}} \|g_t\|^2_\star \frac{\eta^+_t}{\alpha_{t-1}}&\le 15\Lm
\end{align*}
\end{enumerate}
\end{Lemma}

\begin{proof}

\begin{enumerate}
\item
Using part 1 from Proposition \ref{thm:etarates}, and observing that $\eta^+_t\ge \eta_t$, we have
\begin{align*}
    \sum_{t|\ \|g_t\|_\star\le 2L_{t-1}} \|g_t\|^2_\star \eta^+_t&\le \sum_{t|\ \|g_t\|_\star\le 2L_{t-1}} 2\|g_t\|_\star \min(\|g_t\|_\star,L_{t-1}) \eta^+_t\\
    &\le \sum_{t|\ \|g_t\|_\star\le 2L_{t-1}} 2\left(\frac{1}{\eta^+_t}-\frac{1}{\eta_{t-1}}\right)\\
    &\le \sum_{t|\ \|g_t\|_\star\le 2L_{t-1}} 2\left(\frac{1}{\eta^+_t}-\frac{1}{\eta^+_{t-1}}\right)\\
    &\le 2\eta^+_T
\end{align*}

\item
For the second part of the lemma, we observe that for $\|g_t\|_\star \le 2L_{t-1}$, 
\begin{align*}
    \frac{1}{(\eta^+_t)^2} &\ge \frac{1}{(\eta_{t-1})^2}+2\|g_t\|_\star\min(L_{t-1},\|g_t\|_\star)\\
    &\ge \frac{1}{(\eta_{t-1})^2}+\|g_t\|^2_\star\\
    &\ge (\|g\|^2_\star)_{1:t}
\end{align*}

Similarly, we also have $(\|g\|^2_\star)_{1:t}\le (1+\frac{L_{t}^2}{L_{t-1}^2})(\|g\|^2_\star)_{1:t-1}$ so that
\begin{align*}
    \frac{1}{\alpha_{t-1}}&\le L_{t-1}^2\eta_{t-1}^2\\
    &\le \frac{L_{t-1}^2}{2(\|g\|^2_\star)_{1:t-1}}\\
    &\le \frac{L_{t-1}}{L_t}\frac{L_t^2}{2(\|g\|^2_\star)_{1:t-1}}\\
    &\le \frac{L_{t-1}}{L_t}\left(1+\frac{L_{t}^2}{L_{t-1}^2}\right)\frac{L_t^2}{2(\|g\|^2_\star)_{1:t}}\\
    &=\left(\frac{L_{t-1}}{L_t}+\frac{L_{t}}{L_{t-1}}\right)\frac{L_t^2}{2(\|g\|^2_\star)_{1:t}}\\
    &\le \frac{5}{4}\frac{L_t^2}{(\|g\|^2_\star)_{1:t}}\\
\end{align*}
where in the last line we have used $L_t/L_{t-1}\le 2$.

Combining these two calculations, we have
\begin{align*}
    \sum_{t|\ \|g_t\|_\star\le 2L_{t-1}}\|g_t\|^2_\star \frac{\eta^+_t}{\alpha_{t-1}}&\le \frac{5}{4}\sum_{t|\ \|g_t\|_\star\le 2L_{t-1}}\frac{\|g_t\|^2_\star L_{t}^2}{(\|g\|^2_\star)_{1:t}^{3/2}}
\end{align*}

Let $T_1,T_2,\dots, T_n$ be the indices such that $\|g_{T_i}\|_\star >2L_{T_i-1}$, and define $T_n=T+1$. We will show that for any $i$ with $T_{i+1}>T_i+1$,
\begin{align}\label{eqn:epochsum}
    \sum_{t=T_i+1}^{T_{i+1}-1}\frac{\|g_t\|^2_\star L_{t}^2}{(\|g\|^2_\star)_{1:t}^{3/2}} &\le 6L_{T_{i+1}-1}
\end{align}

Observe that for $N=T_i+1$, we have
\begin{align}\label{eqn:sumconverge}
    \sum_{t=T_i+1}^{N}\frac{\|g_t\|^2_\star L_{t}^2}{(\|g\|^2_\star)_{1:t}^{3/2}} &\le 6L_N - \frac{2L_N^2}{\sqrt{(\|g\|_\star^2)_{1:N}}}
\end{align}
We'll prove by induction that equation (\ref{eqn:sumconverge}) holds for all $N\le T_{i+1}-1$. Suppose it holds for some $N< T_{i+1}-1$. Then by concavity of $-\frac{1}{\sqrt{x}}$, we have
\begin{align*}
    \left(6L_{N+1}-\frac{2L_{N+1}^2}{\sqrt{(\|g\|^2_\star)_{1:N+1}}}\right)-\left(6L_{N+1}-\frac{2L_{N+1}^2}{\sqrt{(\|g\|^2_\star)_{1:N}}}\right)&\ge \frac{\|g_{N+1}\|^2_\star L_{N+1}^2}{(\|g\|^2_\star)_{1:N+1}^{3/2}}
\end{align*}

So using the inductive hypothesis:
\begin{align*}
    \sum_{t=1}^{N+1}\frac{\|g_t\|^2_\star L_{t}^2}{(\|g\|^2_\star)_{1:t}^{3/2}}&\le \left(6L_N-\frac{2L_N^2}{\sqrt{(\|g\|^2_\star)_{1:N}}}\right)+\frac{\|g_{N+1}\|^2_\star L_{N+1}^2}{(\|g\|^2_\star)_{1:N+1}^{3/2}}\\
    &=\left(6L_{N+1}-\frac{2L_{N+1}^2}{\sqrt{(\|g\|^2_\star)_{1:N}}}\right)+\frac{\|g_{N+1}\|^2_\star L_{N+1}^2}{(\|g\|^2_\star)_{1:N+1}^{3/2}} + 6(L_N-L_{N+1}) - \frac{2(L_{N}^2-L_{N+1}^2)}{\sqrt{(\|g\|^2_\star)_{1:N}}}\\
    &\le 6L_{N+1}-\frac{2L_{N+1}2}{\sqrt{(\|g\|^2_\star)_{1:N+1}}} + 6(L_N-L_{N+1}) - \frac{2(L_{N}^2-L_{N+1}^2)}{\sqrt{(\|g\|^2_\star)_{1:N}}}
\end{align*}
To finish the induction, we show that $6(L_N-L_{N+1}) - \frac{2(L_{N}^2-L_{N+1}^2)}{\sqrt{(\|g\|^2_\star)_{1:N}}}\le 0$. We factor out the non-negative quantity $L_{N+1}-L_N$, and then observe that $L_{N+1}\le 2L_N$ since $T_i+1\le N< N+1\le T_{i+1}-1$ (and in particular, $L_{N+1}\ne T_i$ for any $i$).
\begin{align*}
    -6 + \frac{2(L_{N}+L_{N+1})}{\sqrt{(\|g\|^2_\star)_{1:N}}}&\le -6+\frac{6L_N}{\sqrt{(\|g\|^2_\star)_{1:N}}}\\
    &\le 0
\end{align*}

Therefore equation (\ref{eqn:sumconverge}) holds for all $N\le T_{i+1}-1$, so that we have
\begin{align}%\label{eqn:sumconverge}
    \sum_{t=T_i+1}^{T_{i+1}-1}\frac{\|g_t\|^2_\star L_{t}^2}{(\|g\|^2_\star)_{1:t}^{3/2}} &\le 6L_{T_{i+1}-1} - \frac{2L_{T_{i+1}-1}^2}{\sqrt{(\|g\|_\star^2)_{1:T}}}\le 6L_{T_{i+1}-1}
\end{align}

so that equation (\ref{eqn:epochsum}) holds. Now we write (using the convention that $\sum_{t=x}^z y_t=0$ if $z< x$):
\begin{align*}
    \sum_{t|\ \|g_t\|_\star\le 2L_{t-1}}\frac{\|g_t\|^2_\star L_{t}^2}{(\|g\|^2_\star)_{1:t}^{3/2}}& = \sum_{i=1}^{n+1} \sum_{t=T_i+1}^{T_{i+1}-1}\frac{\|g_t\|^2_\star L_{t}^2}{(\|g\|^2_\star)_{1:t}^{3/2}}\\
    &\le \sum_{i=1}^{n+1} 6L_{T_{i+1}-1}\\
    &\le 12 \Lm
\end{align*}
where in the last step we have observed that by definition of $T_i$, $L_{T_{i+1}-1}\ge 2 L_{T_i-1}$ for all $i$ and used Lemma \ref{thm:doublingsum}.

Finally, we conclude
\begin{align*}
    \sum_{t|\ \|g_t\|_\star\le 2L_{t-1}}\|g_t\|^2_\star \frac{\eta^+_t}{a_t}&\le \frac{5}{4}\sum_{t=1}^T\frac{\|g_t\|^2_\star L_t^2}{(\|g\|^2_\star)_{1:t}^{3/2}}\\
    &\le 15\Lm
\end{align*}

\end{enumerate}
\end{proof}

\begin{Lemma}\label{thm:asandwich}
Let $\alpha_t$ be defined by
\begin{align*}
    \alpha_0 &= \frac{1}{(L_1\eta_1)^2}\\
    \alpha_t &= \max\left(\alpha_{t-1},\frac{1}{(L_t\eta_t)^2}\right)
\end{align*}
Then
\begin{align*}
    \frac{2(\|g\|_\star)_{1:t}}{L_t}\ge a_t\ge \frac{2(\|g\|^2_\star)_{1:t}}{L_t^2}
\end{align*}
\end{Lemma}
\begin{proof}
Since $\frac{1}{\eta_t^2}\ge 2(\|g\|^2_\star)_{1:t}$, we immediately recover the lower bound on $a_t$. The upper bound follows from Proposition \ref{thm:etarates} (part 2), which states $\frac{1}{\eta_t^2}\le 2L_t(\|g\|_\star)_{1:t}$
\end{proof}

Now we're ready to prove Theorem \ref{thm:parameterfreeregret}, which we restate for reference:
\parameterfreeregret*

\begin{proof}
Using Theorem \ref{thm:ftrlmagic} and Lemmas \ref{thm:norecentering} and \ref{thm:sumbounds}, our regret is bounded by
\begin{align*}
    R_T(u)&\le \psi^+_T(u) + \sum_{t=1}^T \psi_{t-1}(w^+_{t+1})-\psi_t^+(w_{t+1}^+)+g_t(w_t-w^+_{t+1})\\
    &\quad\quad+ \sum_{t=1}^T \psi_t^+(w^+_{t+2})-\psi_t(w^+_{t+2})\\
    &\le \psi^+_T(u)+\sum_{t=1}^T \psi_{t-1}(w^+_{t+1})-\psi_t^+(w_{t+1}^+)+g_t(w_t-w^+_{t+1})\\
    &\le \psi^+_T(u)+\sum_{\|g_t\|_\star \le 2L_{t-1}}\frac{3\|g_t\|^2\eta^+_t}{a_{t-1}\sigmamin}+\sum_{\|g_t\|_\star>2L_{t-1}} \|g_t\|_\star D'
\end{align*}
where $D'$ is defined by
\[
D' = 2\max_t \frac{h^{-1}\left(5\frac{L_t}{kL_{t-1}}\right)}{a_{t-1}}
\]
Now we use Lemma \ref{thm:asandwich} to conclude that
\[
D'\le D = \max_t \frac{L_{t-1}^2}{(\|g\|_\star^2)_{1:t-1}}h^{-1}\left(5\frac{L_t}{kL_{t-1}}\right)
\]
so that we have
\begin{align*}
    R_T(u)&\le \psi^+_T(u)+\sum_{\|g_t\|_\star \le 2L_{t-1}}\frac{3\|g_t\|^2\eta^+_t}{a_{t-1}\sigmamin}+\sum_{\|g_t\|_\star>2L_{t-1}} \|g_t\|_\star D
\end{align*}
Now using Lemma \ref{thm:getasum} we can simplify this to
%\begin{align*}
%    R_T(u) &\le \frac{1}{a\eta^+_T}\left(k\psi(au) +\frac{6}{\sigmamin}\right)+\sum_{\ \|g_t\|_\star>2L_{t-1}} \|g_t\|_\star D
%\end{align*}
%When $a_t=a$ for all $t$ and

\begin{align*}
    R_T(u) &\le \frac{k}{a_T \eta^+_T}\psi(a_Tu) +\frac{45\Lm}{\sigmamin}+\sum_{\ \|g_t\|_\star>2L_{t-1}} \|g_t\|_\star D
\end{align*}
%when $a_t$ is defined via the second recursive definition.

Finally, observe that each value of $\|g_t\|_\star$ in the sum $\sum_{\|g_t\|_\star >2L_{t-1}} \|g_t\|_\star D$ is at least twice the previous value, so that by Lemma \ref{thm:doublingsum} we conclude
%\begin{align*}
%    R_T(u) &\le \frac{1}{a\eta^+_T}\left(k\psi(au) +\frac{6}{a\sigmamin}\right)+2\Lm D
%\end{align*}
%when $a_t=a$ for all $t$ and
\begin{align*}
    R_T(u) &\le \frac{k}{a_T \eta^+_T}\psi(a_Tu) +\frac{45\Lm}{\sigmamin}+2\Lm D
\end{align*}

Finally, we observe that (by Lemma \ref{thm:asandwich}), $a_T\le 2\frac{\|g\|_{1:T}}{L_T}=Q_T$, which gives the first inequality in the Theorem statement.
%when $a_t$ is defined using the second recursive definition.

Using the fact that $\frac{1}{\eta_t}\le \sqrt{2\Lm(\|g\|_\star)_{1:t}}$ (from Proposition \ref{thm:etarates} part 2), we have $\eta^+_T \ge \frac{1}{\Lm \sqrt{2T}}$ and it is clear that $a_T\le2T$, so that we recover the second inequality as well.

\end{proof}

\end{document}